\documentclass[
  journal=psrm,
  manuscript=research-note,  
  year=2021,
  volume=2,
]{cup-journal}
\DeclareUnicodeCharacter{0301}{\'{e}}
\usepackage{booktabs,microtype,siunitx}
\usepackage{amsfonts,amsmath}
\usepackage{amssymb}
\usepackage{multirow}
\usepackage{tikz}
\usepackage{tcolorbox}
\usepackage{wrapfig}
\usepackage{amsthm}
\newtheorem{theorem}{Theorem}
\newtheorem{remark}{Remark}

\newtheorem{lemma}{Lemma}
\newtheorem{assumption}{Assumption}
\newtheorem{definition}{Definition}

\usepackage{amsmath,amsfonts,bm}

\def\eqref#1{equation~\ref{#1}}

\def\1{\bm{1}}

\def\eps{{\epsilon}}

\DeclareMathAlphabet{\mathsfit}{\encodingdefault}{\sfdefault}{m}{sl}
\SetMathAlphabet{\mathsfit}{bold}{\encodingdefault}{\sfdefault}{bx}{n}

\newcommand{\E}{\mathbb{E}}

\DeclareMathOperator*{\argmin}{arg\,min}

\usepackage{amssymb,enumerate,graphicx}
\usepackage[unicode=true,bookmarksnumbered=true,colorlinks]{hyperref}
\definecolor{amethyst}{rgb}{0.6, 0.4, 0.8}
\hypersetup{
    colorlinks=true,
    linkcolor=purple,
    filecolor=magenta,      
    urlcolor=cyan,
    citecolor=amethyst,
}
\usepackage{url}
\usepackage{tikz}
\usepackage{tcolorbox}
\usepackage{wrapfig}
\usepackage{subcaption}
\newtcbox{\alertinline}[1][red]
  {on line, arc = 0pt, outer arc = 0pt,
    colback = #1!20!white, colframe = #1!50!black,
    boxsep = 0pt, left = 1pt, right = 1pt, top = 2pt, bottom = 2pt,
    boxrule = 0pt, bottomrule = 1pt, toprule = 1pt}
\newtcolorbox{mybox}{colframe = red!50!black}
\usepackage[nameinlink,noabbrev,capitalise]{cleveref}
\Crefname{table}{Table}{Tables}
\crefname{equation}{}{}
\Crefname{assumption}{Assumption}{Assumptions}

\newtheorem{corollary}{Corollary}[section]
\newtheorem{example}{Example}[section]

\newcommand{\A}{\mathcal{A}}
\renewcommand{\O}{\mathcal{O}}
	
\newenvironment{proof-sketch}{\noindent\textbf{\textit{Proof sketch}:}}{\hfill$\square$}

\newcommand{\widesim}[2][1.5]{
	\mathrel{\overset{#2}{\scalebox{#1}[1]{$\sim$}}}
}
\newcommand{\iidsim}{\widesim[2.33]{\mathrm{i.i.d.}}}	
\renewcommand{\E}{\mathbb{E}}					

\renewcommand{\A}{\mathcal{A}}
\renewcommand{\O}{\mathcal{O}}
\renewcommand{\H}{\mathcal{H}}

\newcommand{\T}{\mathcal{T}}
\newcommand{\HS}{\mathrm{HS}}

\newcommand{\tr}[1]{\operatorname{tr}\left(#1\right)}		

\renewcommand{\L}{\mathcal{L}}
\renewcommand{\argmin}{\operatornamewithlimits{arg\,min}}

\renewcommand{\geq}{\geqslant}
\renewcommand{\leq}{\leqslant}
\renewcommand{\eps}{\varepsilon}

\newcommand{\C}{\mathcal{C}}
\newcommand{\half}{\frac{1}{2}}

\newcommand\normx[1]{\left\Vert#1\right\Vert}

\title{Minimax Optimal Kernel Operator Learning via Multilevel Training}

\author{Jikai Jin}
\affiliation{School of Mathematical Sciences, Peking University, BeiJing, China}
\email{jkjin@pku.edu.cn}

\author{Yiping Lu}
\affiliation{ICME, Stanford University, CA, USA}
\email{yplu@stanford.edu}

\author{Jose Blanchet}
\affiliation{Management Science \& Engineering, Stanford University, CA, USA}
\email{jose.blanchet@stanford.edu}

\author{Lexing Ying}
\affiliation{Department of Mathematics, Stanford University, Stanford, CA, USA}
\email{lexing@stanford.edu}

\keywords{Reproducing Kernel Hilbert Space, Minimax Optimal Rate, Operator Learning, Kernel Mean Embedding} 

\begin{document}

\begin{abstract}
Learning mappings between infinite dimensional function spaces has achieved empirical success in many disciplines of machine learning, including generative modeling, machine learning solving Partial Difference Equations , functional data analysis, causal inference, and multi-agent reinforcement learning. In this paper, we study the statistical limit of learning a Hilbert-Schmidt operator between two infinite-dimensional Sobolev reproducing kernel Hilbert spaces. We establish the information-theoretic lower bound in terms of the Sobolev Hilbert-Schmidt norm and show that a regularization that learns the spectral components below the bias contour and ignores the ones that above the variance contour can achieve the optimal learning rate. At the same time, the spectral components between the bias and variance contours give us the flexibility in designing computationally feasible machine learning algorithms. Based on this observation, we develop a multilevel kernel operator learning algorithm that is optimal when learning linear operators between infinite-dimensional function spaces.
\end{abstract}

\section{Introduction}
\label{sec-intro}
The supervised learning of operators between two infinite-dimensional spaces has attracted 
attention in many machine learning applications, such as  scientific computation \citep{lu2019deeponet,li2020fourier,de2021convergence,li2018organic,li2021principles}, functional data analysis \citep{crambes2013asymptotics,hormann2015note,wang2020functional}, learning mean-field games \citep{guo2019learning,wang2020breaking}, conditional probability regression \citep{song2009hilbert,song2013kernel,muandet2017kernel} and econometrics \citep{singh2019kernel,muandet2020dual,dikkala2020minimax}. Despite the empirical success of operator learning, the statistical limit of learning an infinite-dimensional operator is poorly studied.  In this paper, we study the problem of learning Hilbert Schmidt operators between infinite-dimensional Sobolev reproducing kernel Hilbert spaces $\mathcal{H}_K^\beta$ and $\mathcal{H}_L^\gamma$ with given kernels $k$ and $l$ respectively and $\beta,\gamma\in[0,1)$ \citep{adams2003sobolev,christmann2008support,fischer2020sobolev}. Our goal is to derive the optimal sample complexity to learn the linear operator, \emph{i.e.} how much data is required to achieve a certain performance level. 

We first establish an information-theoretic lower bound for learning a Hilbert-Schmidt operator between Sobolev spaces respect to a general Sobolev norm. Our information-theoretic lower bound indicates that the optimal learning rate is determined by the minimum of two polynomial rates: one is purely decided by the input Sobolev reproducing kernel Hilbert space and its evaluating norm, while the other one is purely determined by the output space along with its evaluating norm. The rate is novel in the sense that all existing results \citep{fischer2020sobolev,li2022optimal,de2021convergence} only establish rates that depend on the parameter of input space. The reason is all previous works \citep{talwai2022sobolev,li2022optimal,de2021convergence} only consider the case of the output space as a subspace of a trace bounded reproducing kernel Hilbert space but not a general Sobolev space.  We refer to  \Cref{remark:assumption} for detail comparisons.

To design a learning algorithm for approximating an infinite-dimensional operator, we need to learn a finite-dimensional restriction instead of the whole operator, as the latter would result in infinite variance. The finite-dimensional selection leads to bias error but decreases the variance. A natural task is then to study the shape of regularization that can lead to the optimal bias-variance trade-off and achieve the optimal learning rate. In this paper, we consider the bias and variance contour at the scale of optimal learning. Once the regularization enables one to learn all the spectral parts above the bias contour and below the variance contour, the learning is optimal. Finally, utilizing the region between the bias contour and variance contour, we developed a multilevel training algorithm \citep{lye2021multi,li2021multilevel} which first learns the mapping on low frequency and then successively fine-tunes the machine learning models to fit the high-frequency output. The intuition of our algorithm aligns with the original motivation of multilevel Monte Carlo \citep{giles2008multilevel,giles2015multilevel}: we use the next level to reduce bias while keeping the variance at the same scale. We demonstrate that such a multilevel algorithm can achieve an optimal non-parametric rate for linear operator learning.

\subsection{Related Work}
\label{subsec:related-work}
\paragraph{Machine Learning Based PDE Solver} Solving partial differential equations (PDEs) plays a prominent role in many scientific and engineering discipline, such as physics, chemistry, operation management, macro-economy, etc. The recent deep learning breakthrough has drawn attention to solving PDEs via machine learning methods \citep{raissi2019physics,han2018solving,sirignano2018dgm,yu2018deep,khoo2019solving,chen2021solving}. The statistical power and computational cost of these problem is well-studied by recent papers \citep{lu2021machine,lu2022sobolev,nickl2020convergence,nickl2020polynomial}. This paper focuses on operator learning \citep{chen1995universal,long2018pde,long2019pde,feliu2020meta,khoo2021solving,lu2019deeponet,li2020fourier,kovachki2021neural,stepaniants2021learning}, \emph{i.e.} learning a map between two infinite dimensional function spaces. For example, one can learn a PDE solver that maps from the boundary condition to the solution or an inverse problem that maps from the boundary measurement to the coefficient field. In terms of the mathematical foundation of operator learning, \citep{liu2022deep} considers the learning rate of non-parametric operator learning. However, non-parametric functional data analysis often suffers from slower-than-polynomial convergence rates \citep{mas2012lower}, due to the small ball probability problem for the probability distributions in infinite dimensional spaces \citep{delaigle2010defining}. The most relevant works are \citep{lin2011fast,reimherr2015functional,de2021convergence}, which consider the rates for learning a linear operator. For the comparison between our work and \citep{de2021convergence}, see \Cref{remark:assumption}.

\paragraph{Learning with kernel.} Supervised least square regression in RKHS and its generalization capability have been thoroughly studied \citep{caponnetto2007optimal,smale2007learning,de2005learning,rosasco2010learning,mendelson2010regularization}. The minimax optimality with respect to the Sobolev norm has been discussed recently in \citep{fischer2020sobolev,liu2020estimation,lu2022sobolev}. Our paper is highly related to recent works \citep{schuster2020kernel,mollenhauer2020nonparametric,talwai2022sobolev,li2022optimal} on identifying the Sobolev norm learning rate for the kernel mean embedding\citep{song2009hilbert,song2013kernel,muandet2017kernel,park2020measure,singh2020kernels}, which can also formulated as learning an operator. The difference of our work see \Cref{remark:assumption}.

\paragraph{Multilevel Monte Carlo} By combining biased estimators with multiple stepsizes, multilevel Monte Carlo (MLMC) \citep{giles2008multilevel,giles2015multilevel} dramatically improves the rate of convergence and achieves in many settings the canonical square root convergence rate associated with unbiased Monte Carlo \citep{rhee2015unbiased,blanchet2015unbiased}. Multilevel Monte Carlo can also be used for random variable with infinite variance \citep{blanchet2016malliavin,chen2020unbiased}. To the best of our knowledge, this is the first paper that provides optimal sample complexity for multilevel Monte Carlo type algorithm for infinite variance problems in the non-parametric regime. Very recently,  \citep{lye2021multi,li2021multilevel} developed a multilevel machine learning Monte Carlo algorithm (ML2MC) / multilevel fine-tuning algorithm for learning solution maps, by first learning the map on coarsest grid and then successively fine-tuning the network on samples generated at
finer grids. The authors also showed that, following the telescoping in MLMC, the multilevel training procedure can reduce the generalization error without spending more time on generating training samples. \citep{schafer2021sparse,boulle2022learning} consider such multi-scale algorithm for learning Green's function. However, the statistical power of such algorithm is still under investigation. Another difference with \citep{boulle2022learning} is that we consider the Green function in $H^{-1}$ norm rather than the $\ell_1$ norm used in \citep{boulle2022learning}. In this paper, we qualify a specific setting where this multilevel procedure can and is necessary to achieve the minimax optimal learning rate.

\subsection{Contribution}
\begin{itemize}
    \item We derive a novel information-theoretic lower bound of learning a linear operator between two infinite-dimensional Sobolev reproducing kernel Hilbert spaces. The optimal learning rate is a minimum of two polynomial rates, one only dependent on the parameters of the input space while the other only on the parameters of the output space. The first rate aligns with the previous works \citep{li2022optimal}, while the second lower bound is novel to the literature.
    \item We study the shape of regularization that can lead to the optimal learning rate. One should learn all the spectral parts under the bias contour at the level of the optimal learning rate but not the spectral parts above the variance contour at the level of learning rate. This enables the estimator to enjoy an optimal balance of bias-variance.
    \item We qualify a specific setting where a multilevel training procedure \citep{lye2021multi,li2021multilevel} is necessary and capable of achieving a minimax optimal learning rate for learning a linear operator. We achieve the optimal learning rate via $O(\ln \ln n)$ ensemble of ridge regression models. This is different from finite-dimensional operator learning where a single level estimator can be optimal.
\end{itemize}

\section{Problem Formulation}
\label{sec-problem}

\subsection{Preliminary}


Let $P_K$ be a distribution over the input space $\H_K$ and define covariance operator $\C_{KK}= \E_{u\sim P_K} u\otimes u$. Consider its spectral decomposition $\C_{KK} = \sum_{i=1}^{+\infty} \mu_i^2 e_i \otimes e_i$, where $\{\mu_i^{\half}e_i\}_{i=1}^{+\infty}$ is an orthogonal eigenbasis and $\{\mu_i\}$ is the corresponding eigenvalues of $\C_{KK}$ (here the $g\otimes h$ is an operator defined as $g\otimes h=gh^\ast:f\rightarrow \left\langle f,h\right\rangle g$). In the typical machine learning applications, the test distribution is the same as the training distribution, so we can assume that $\H_K = \left\{ \sum_{i} a_i \mu_i^{\half}e_i: \{a_i\}_{i=1}^{\infty}\in\ell_2 \right\}$ without loss of generality. Note that this automatically holds in the context of learning the conditional mean embedding (CME) \citep{fischer2020sobolev,talwai2022sobolev,li2022optimal}.


Following \citep{christmann2008support,fischer2020sobolev}, we define the interpolation Sobolev space $\H_K^\beta=\left\{f=\sum_i a_i(\mu_i^{\frac{\beta}{2}}e_i):\{a_i\}_{i=1}^\infty\in l^2\right\}$ for any $\beta>0$, equipped with Sobolev norm defined by the inner product $\left<\sum_i a_i(\mu_i^{\beta/2}e_i),\sum_i b_i(\mu_i^{\beta/2}e_i)\right>_{\mathcal{H}_K^\beta}=\sum_i a_i b_i$. For the output space, we fix a user-specified distribution $Q_L$ and a reproducing Kernel Hilbert Space. We can similarly define the covariance operator $\C_{Q_L}$ and the Sobolev space $\H_L^{\gamma}$. Natural choices of $Q_L$ include some distribution on kernel functions $\left\{\ell(y,\cdot):y\in Y\right\}$ of $\H_L$ induced by some  distribution $Q_L$ on $Y$, so that $\C_{Q_L}$ is a kernel integral operator with respect to $Q_L$ and $\H_L^{\gamma}$ is an interpolation space between $\H_L$ and $\L^2(Q_L)$; see \Cref{ex-pde} for a specific choice of $Q_L$ and its practical implications.

Following \citep{li2022optimal}, in this paper we consider the Hilbert-Schmidt norm between two Sobolev Spaces for all the operators, which is defined as following.

\begin{definition}[$(\beta,\gamma)$-norm]
\label{def:betagammanorm}
    Let $T:\H_K\mapsto\H_L$ be a possibly unbounded linear operator. $I_{1,\beta,P_{K}}:\H_K\mapsto\H_K^{\beta}, \beta\in(0,1)$ is the canonical embedding mapping that takes $u\in\H_K$ to the same element $u$ in the larger space $\H_K^{\beta}$, and $I_{1,\gamma,Q_L}:\H_L\mapsto\H_L^{\gamma},\gamma\in(0,1)$ is similarly defined. Then the $(\beta,\gamma)$-norm of $T$ is defined as
    \begin{equation}
        \notag
        \normx{T}_{\beta,\gamma} = \normx{(I_{1,\gamma,Q_{L}}^\ast)^{\dagger}\circ T \circ I_{1,\beta,P_{K}}^\ast}_{\HS\left(\H_K^{\beta},\H_{L}^{\gamma}\right)}=  \normx{\mathcal{C}_{Q_L}^{-(1-\gamma)/2}\circ T\circ \mathcal{C}_{KK}^{(1-\beta)/2}}_{\HS\left(\H_K,\H_{L}\right)},
    \end{equation}
    where we omit the dependence of $\normx{\cdot}_{\beta,\gamma}$ on $P_{K}$ and $Q_L$ since it will always be clear from context.
\end{definition}

\subsection{Problem Formulation}
We consider the problem of learning an unknown linear operator $\A_0:\H_K\mapsto\H_L$ between two reproducing kernel Hilbert spaces corresponding to kernl $k$ and $l$ respectively. We are given $N$ noisy data pairs $(u_i,v_i),1\leq i\leq N$ related by 
\begin{equation}
    \label{data-model}
    v_i = \A_0 u_i + \eps_i
\end{equation}
where $u_i\iidsim P_{K}$ for some unknown distribution $P_{K}$ and $\eps_i$ is the noise drawn from some distribution with zero mean that may depend on $u_i$. 
We use $P_{KL}$ for the joint distribution of $(u_i,v_i)$. Denote $\C_{KK} = \E_{u\sim P_K}u\otimes u$,  $\C_{KL}=\E_{(u,v)\sim P_{KL}}u\otimes v$ and its adjoint $\C_{LK}=\C_{KL}^\ast$ be uncentered cross-covariance operators associated with $P_{KL}$. Then we can reformulate the ground turth operator as  $\A_0=\C_{LK}\C_{KK}^{\dagger}$, where $\dagger$ is the pseudo-inverse \citep{talwai2022sobolev,li2022optimal}. With the goal of understanding the relative difficulty of learning different types of linear operators, we investigate the sample efficiency of learning $\A_0$ under certain source assumptions imposed on the data model \Cref{data-model}. Source condition \citep{caponnetto2007optimal,mendelson2010regularization,steinwart2009optimal,rosasco2010learning,fischer2020sobolev} assumes that the learning target lies in a parameterized function class and study the learning rate for different problems with different hardness. Specifically, the source condition assume that the learning target is bounded in certain Sobolev norm. In this paper, we consider learning an operator with bounded $(\beta,\gamma)$-norm, which is the Hilbert-Schmidt norm that maps from $\mathcal{H}_K^\beta$ to $\mathcal{H}_L^\gamma$. We consider the generalization error/convergence rate under another $(\beta',\gamma')$-norm as in \citep{fischer2020sobolev,lu2022sobolev,talwai2022sobolev,li2022optimal}.

\begin{remark}
\label{remark:assumption} 
Although recent works have considered similar problems in the context of conditional mean embedding  \citep{talwai2022sobolev,li2022optimal} and functional data analysis \citep{de2021convergence}. In all these papers, the output space is a trace bounded reproducing kernel Hilbert space \cite[Assumption 2.14 (vi)]{de2021convergence} rather than the general parameterized Sobolev space in our paper. 
\end{remark}

We then list all the assumptions imposed on the underlying kernel for our theoretical results. We follow the standard capacity assumptions and embedding properties used in kernel regression \citep{fischer2020sobolev,talwai2022sobolev,li2022optimal}.

\begin{assumption}[Capacity Condition of the Covariance]
\label{asmp-eigendecay}
The eigenvalues $\{\mu_i\}_{i\geq 1}$ of the covariance operator $\mathcal{C}_{KK} = \E_{u\sim P_{K}}u\otimes u$ satisfies $\mu_i\propto i^{-\frac{1}{p}}$ for some $p\in(0,1)$. Similarly, the eigenvalues $\{\rho_i\}_{i\geq 1}$ of the covariance operator $\mathcal{C}_{Q_L} = \E_{v\sim Q_L}v\otimes v$ satisfies $\rho_i\propto i^{-\frac{1}{q}}$ for some $q\in(0,1)$.
\end{assumption}



\begin{assumption}[$\ell_\infty$ Embedding Property of the Input RKHS]
    \label{asmp-bounded-input-data}
    There exists a smallest $\alpha\in(0,1)$ such that $\normx{\left(I_{1,\alpha,P_{K}}^\ast\right)^{\dagger}f}_{\H_K^{\alpha}} \leq A_1$ a.s. under $P_{K}$ for some $A_1<+\infty$.
\end{assumption}

\begin{assumption}[$\ell_\infty$  Embedding Property of the Output RKHS]
    \label{asmp-bounded-output-data}
    There exists $A_2 < +\infty$ such that $\normx{\A_0 u}_{\H_{L}} \leq A_2$ holds for all $u \in \mathrm{supp}(P_K)$, except from a $P_K$-null set.
\end{assumption}

\begin{assumption}[Moment Condition]
\label{asmp-noise}
There exists an operator $V:\H_{L}\mapsto\H_{L}$ with $\tr{V} \leq \sigma^2$ such that for every $u \in \mathrm{supp}(P_K)$, we have  $\E_{v\sim P_{KL}(\cdot\mid u)}\left[ \left((v-\A_0 u) \otimes (v-\A_0 u)\right)^k \right] \preceq \frac{1}{2}(2k)! R^{2k-2}V$ holds for all $k\ge 2$.
\end{assumption}


\begin{assumption}[Source Condition]
\label{asmp-beta-gamma-norm}
 $\A_0$ is bounded under $(\beta,\gamma)$-norm i.e. $\normx{\A_0}_{\beta,\gamma} \leq B$ for some $B\in(0,+\infty)$.
\end{assumption}

\subsection{Examples}
\label{sec-example}
In this section, we will introduce two examples of our theory.  The first one is about learning a differential operator, for example inferring an advection-diffusion model  \citep{portone2022bayesian} from observations or predicting the future \citep{long2018pde,lu2019deeponet,li2020fourier,feliu2020meta,huang2021meta}. The second example is about learning conditional mean embedding \citep{song2009hilbert,song2013kernel,muandet2017kernel}, which represents a conditional distribution as an RKHS element. Thus conditional distribution regression can be reduced to a kernel operator learning. Our theory can also be used for linear inverse problem such as radial electrical impedance tomography (EIT) \citep{mueller2012linear} and the severely ill-posed inverse boundary problem for the Helmholtz
equation with unknown wave-number parameter \citep{agapiou2014bayesian}. For detailed discussion, we refer to \citep[Section 1.3]{de2021convergence}

\begin{example}[Learning differential operators]
\label{ex-pde}
Suppose that the ground-truth operator $\A_0 = \Delta^{t}$ where $\Delta$ is the Laplacian and $t \in\mathbb{Z}$. Let $\H_K = \H^{m+2t}([0,1])$ be the Sobolev space with smoothness $m+2t$ on $[0,1]$ and $\H_L = \H^{m}([0,1])$, then $\A_0$ is a bounded operator from $\H_K$ to $\H_L$ which corresponds to the $\beta=\gamma=1$ case. However, we will see below that we can obtain a better characterization of the learning error using our theory.

Consider for example that the input has mean zero and the Mat{\'e}rn-type covariance operator $C_{KK} = \sigma^2\left(-\Delta+\tau^2 I\right)^{-s}$. Its eigenvalues satisfy $\mu_n \propto n^{-2s}$. On the other hand, we choose $Q_L$ to be a distribution supported on $\left\{\ell(y,\cdot):y\in[0,1]\right\}$ induced by a uniform distribution on $[0,1]$, where $\ell$ is the kernel function of $\H_L$. Then $\mathcal{C}_{Q_L}$ is essentially the kernel integral operator on $\H_L$ w.r.t. the uniform distribution, and its eigenvalues are $\rho_n \propto n^{-2m}$. The assumption $||\A_0||_{\beta,\gamma} < +\infty$ is satisfied if and only if $(1-\gamma)m < (1-\beta)s-\frac{1}{2} \Rightarrow \gamma > 1 - \frac{2(1-\beta)s-1}{2m}$.

\end{example}

\begin{example}[Conditional mean embedding]
\label{ex-cme}
Suppose that we would like to learn the conditional distribution $P(y\mid x)$ from a data set $\{(x_i,y_i):1\leq i\leq N\}\subset X\times Y$ where $x_i\iidsim P_X$. Let $\H_K$ and $\H_L$ be two RKHSs on $X$ and $Y$ respectively, with measurable kernel $k(\cdot,\cdot)$ and $\ell(\cdot,\cdot)$. Then we can define a conditional mean embedding (CME) operator $C_{Y\mid X}$ that satisfies
$$
C_{Y\mid X} k(x,\cdot) = \E_{Y\mid x}\ell(Y,\cdot) =: \mu_{Y\mid x}, \text{ and } \E_{Y\mid x} \left[ g(Y) \right] = \left\langle g,\mu_{Y\mid x}\right\rangle \forall x\in X.
$$
We choose $\A_0 = C_{Y\mid X}$. In this case, $\mathcal{C}_{KK} = \E_{P_X} k(X,\cdot)\otimes k(X,\cdot)$. \Cref{asmp-bounded-input-data} states that $\sup_{x\in X}k^{\alpha}(x,x) = A_1$, while \Cref{asmp-bounded-output-data} is equivalent to $\sup_{x\in X}\normx{\mu_{Y\mid x}} \leq A_2$ (for simplicity we only focus on the case $\zeta = 1$). According to \Cref{asmp-beta-gamma-norm}, we assume that $\normx{C_{Y\mid X}}_{\beta,\gamma}\leq B$.

The mis-specified setting where $\beta<1$ has been studied in previous work \citep{fischer2020sobolev,talwai2022sobolev,li2022optimal}. However, they only consider the case $\gamma = 1$. Our results also cover the case $\gamma<1$, which allows us to obtain theoretical guarantee for computing conditional expectation of the larger function class $\H_L^{\gamma}$. 

\end{example}


\section{Information Theoretic Lower Bound}
\label{sec-lower-bound}
In this section, we provide an information-theoretic lower bound for the convergence rate of the operator learning problem formulated in \Cref{sec-problem}.

\begin{theorem}
\label{thm:lowerbound}
    Suppose that $\H_K$ and $\H_L$ are two Hilbert spaces, $P_K$ and $Q_L$ are probability distributions on $\H_K$ and $\H_L$ respectively such that \Cref{asmp-eigendecay,asmp-bounded-input-data} hold. Then for any estimator $\mathcal{L}:\left(\H_K\times \H_L\right)^{\otimes N}\mapsto \HS\left(\H_K^{\beta},\H_L^{\gamma}\right)$, there exists a linear operator $\A_0$ and a joint data distribution $P_{KL}$ with marginal distribution $P_K$ on $\H_K$ satisfying  \Cref{asmp-bounded-output-data,asmp-noise,asmp-beta-gamma-norm}, such that with probability $\geq 0.99$ over $(u_i,v_i)\iidsim P_{KL}$ we have
    \begin{equation}
        \notag
        \normx{\mathcal{L}\left(\{(u_i,v_i)\}_{i=1}^N\right)-\A_0}_{\beta',\gamma'}^2 \gtrsim N^{-\min\left\{\frac{\max\{\alpha,\beta\}-\beta'}{\max\{\alpha,\beta\}+p},\frac{\gamma'-\gamma}{1-\gamma}\right\}}.
    \end{equation}
\end{theorem}


\begin{remark} Our lower bound is composed of a minimum of two parts. The first rate $N^{-\frac{\max\{\alpha,\beta\}-\beta'}{\max\{\alpha,\beta\}+p}}$ is the minimax optimal Sobolev learning rate for kernel regression \citep{fischer2020sobolev,talwai2022sobolev,li2022optimal,lu2022sobolev} and is fully determined  by the parameter of the input Sobolev reproducing kernel Hilbert space. Our second rate $N^{-\frac{\gamma'-\gamma}{1-\gamma}}$ is novel to the literature. This bound shows that how the infinite dimensional problem is different from finite dimensional regression problem and is fully determined  by the parameter of the output Sobolev reproducing kernel Hilbert space. Our lower bound shows that the hardness of learning a linear operator is determined by the harder part between the input and output spaces. We will explain why the lower bound has such structure in \Cref{remark4.1} and \Cref{contours_fig}.
\end{remark}

\section{On the  Shape of Regularization}
\label{sec-estimator}
In this section, we aim to understand the shape of regularization so that the constructed estimator $\hat{\A}$ based on $N$ i.i.d. data $\{(u_i,v_i)\}_{i=1}^n\sim P_{KL}^{\otimes n}$ for $1\leq i\leq N$ enjoys an optimal learning rate. 

Compared with existing approaches where a regularized least-squares estimator can achieve statistical optimality \citep{fischer2020sobolev,talwai2022sobolev,li2022optimal,de2021convergence} under $(\beta,1)$-norm, we study the learning rate under the $(\beta,\gamma)$-norm  ($\beta'\in(0,\beta),\gamma'\in(\gamma,1)$) which is defined in \Cref{def:betagammanorm} as $\normx{\hat{\A}-\A_0}_{\beta',\gamma'}=\normx{\mathcal{C}_{Q_L}^{-\frac{1-\gamma'}{2}}\left(\hat{\A}-\A_0\right)\mathcal{C}_{KK}^{\frac{1-\beta'}{2}}}_{\HS(\H_K,\H_L)}$. The norm of the additional $\mathcal{C}_{Q_L}^{-\frac{1-\gamma'}{2}}$ term is unbounded which make our setting harder than the convergence in $(\beta,1)$-norm in existing works. Since $\mathcal{C}_{Q_L}^{-\frac{1-\gamma'}{2}}$ is bounded when restricted to the finite-dimensional space $\mathrm{span}\left(\rho_i^{\frac{1}{2}}f_i:1\leq i\leq n\right)$, we should also include another bias-variance trade-off via regularizing in the output shape. As a result, we are interested in answering the following question
\begin{center}
    \emph{What is the optimal way to combine the regularization in the input space and regularization in the output space? \emph{i.e.} What is the optimal shape of regularization?} 
\end{center}
To answer this question, we investigate the problem in the spectral space, \emph{i.e.} considering the spectral representation of operator $\A_0=\sum_{i,j=1}^{+\infty}a_{ij}\mu_i^{\frac{\beta}{2}}e_i\otimes \rho_j^{1-\frac{\gamma}{2}}f_j$.  The problem of estimating $\A$ then reduces to learning the coefficients ``matrix" $(a_{ij})_{i,j=1}^\infty$. The source condition \Cref{asmp-beta-gamma-norm} enforces $\sum_{i,j=1}^\infty a_{ij}^2\le B$. We show in \Cref{sec-bias} that regularizing the basis $e_i\otimes f_j$ will introduce a bias of order $\normx{a_{ij}\mu_i^{\frac{\beta}{2}}e_i\otimes \rho_j^{1-\frac{\gamma}{2}}f_j}_{\beta',\gamma'}^2 = a_{ij}^2\mu_i^{\beta-\beta'}\rho_j^{\gamma'-\gamma} \propto i^{-\frac{\beta-\beta'}{p}}j^{-\frac{\gamma'-\gamma}{q}}$ under the $(\beta',\gamma')$-norm. On the other hand, when $\alpha\leq\beta+p$, we show in \Cref{sec-var}  that the variance of learning $(i,j)$ from noisy data scales as $\frac{1}{N}\mu_i^{-\beta'}\rho_j^{-(1-\gamma')} \propto \frac{1}{N}i^{\frac{\beta'}{p}}j^{\frac{1-\gamma'}{q}}$. Since the variance would accumulate for a fixed $j$, learning $(i,j)$ for $i\leq i_{\max}$ results in a variance of $\propto \frac{1}{N}i_{\max}^{\frac{\beta'+p}{p}}j^{\frac{1-\gamma'}{q}}$. (Similar analysis can be carried out for the ${\alpha>\beta+p}$ case as well, but the variance now scales as $\frac{1}{N}i_{\max}^{\frac{\beta'+\alpha-\beta}{p}}j^{\frac{1-\gamma'}{q}}$; see \Cref{appsec:upper-proof} for detailed derivations.) In summary, we need to make bias-variance trade off in the $(i,j)-$plane, \emph{i.e.} decide whether we should learn or regularize over the basis $e_i\otimes f_j$.

\subsection{Intuitive Explanation}

{In this section, we provide an intuitive explanation of our lower bound (Theorem \ref{thm:lowerbound}). As the previous paragraph explains, learning an operator is equivalent to learning an "infinite" size matrix with larger variance and smaller bias on the right upper corner. The core proof of this paper is considering proper bias-variance trade-off. Suppose one wants to construct an estimator with $N^{-\theta}$ learning rate. They need to learn every spectral component below the bias counter at the level of $N^{-\theta}$. Otherwise, the bias itself will become larger than $N^{-\theta}$. At the same time, they still need not learn any spectral component above the variance counter at the level of $N^{-\theta}$. Otherwise, the variance itself will become larger than $N^{-\theta}$. Thus, to enable an feasible estimator achieves learning rate at $N^{-\theta}$, the variance counter at the level of $N^{-\theta}$ should always be above the bias counter at the level of $N^{-\theta}$ (Figure \ref{contours_fig}). Depending on the hyperparameters, there are two different ways to achieve this goal, as shown in Figure \ref{contours_fig}. Each situation is mapped to the rate depending only on the input space, and the rate depends only on the output space. In Section \ref{sec-multilevel}, we demonstrated how a multilevel algorithm could be used to satisfy this requirement.}

\subsection{Regularization via variance contour}
\label{subsec-main-idea}

The underlying idea of regularization is that some components are intrinsically hard to learn due to large variance; these components are then neglected by adding regularization and are counted as bias. The remaining components are easy to learn due to controllable variance. This intuition works well when the estimation error results from the noise of the data and is well-studied in a line of works \citep{fischer2020sobolev,de2021convergence,talwai2022sobolev,li2022optimal}. This idea still works in our setting, but we need to re-evaluate the bias and variance of each component. Since we work with the Hilbert-Schmidt norm, this can be done in a coordinate-wise manner, meaning that we can look at each $a_{ij}$ separately and decide whether to neglect it (contribute to bias) or to learn it from data (contribute to variance).

\begin{figure*}
 
    \centering
    \begin{subfigure}[a]{0.47\textwidth}
        \centering
        \includegraphics[width=\textwidth]{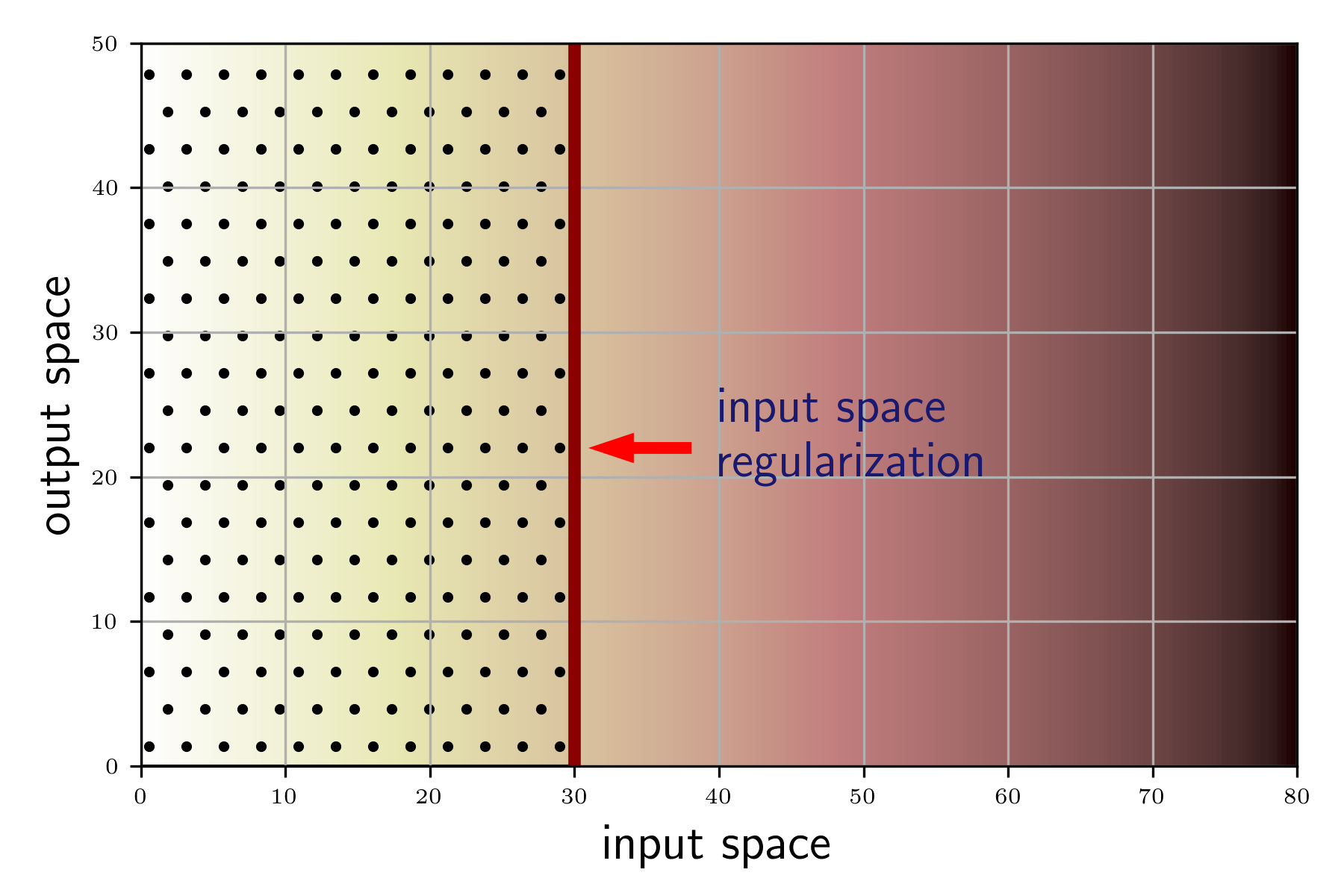}
    \end{subfigure}
    \qquad
    \begin{subfigure}[a]{0.47\textwidth}  
        \centering 
        \includegraphics[width=\textwidth]{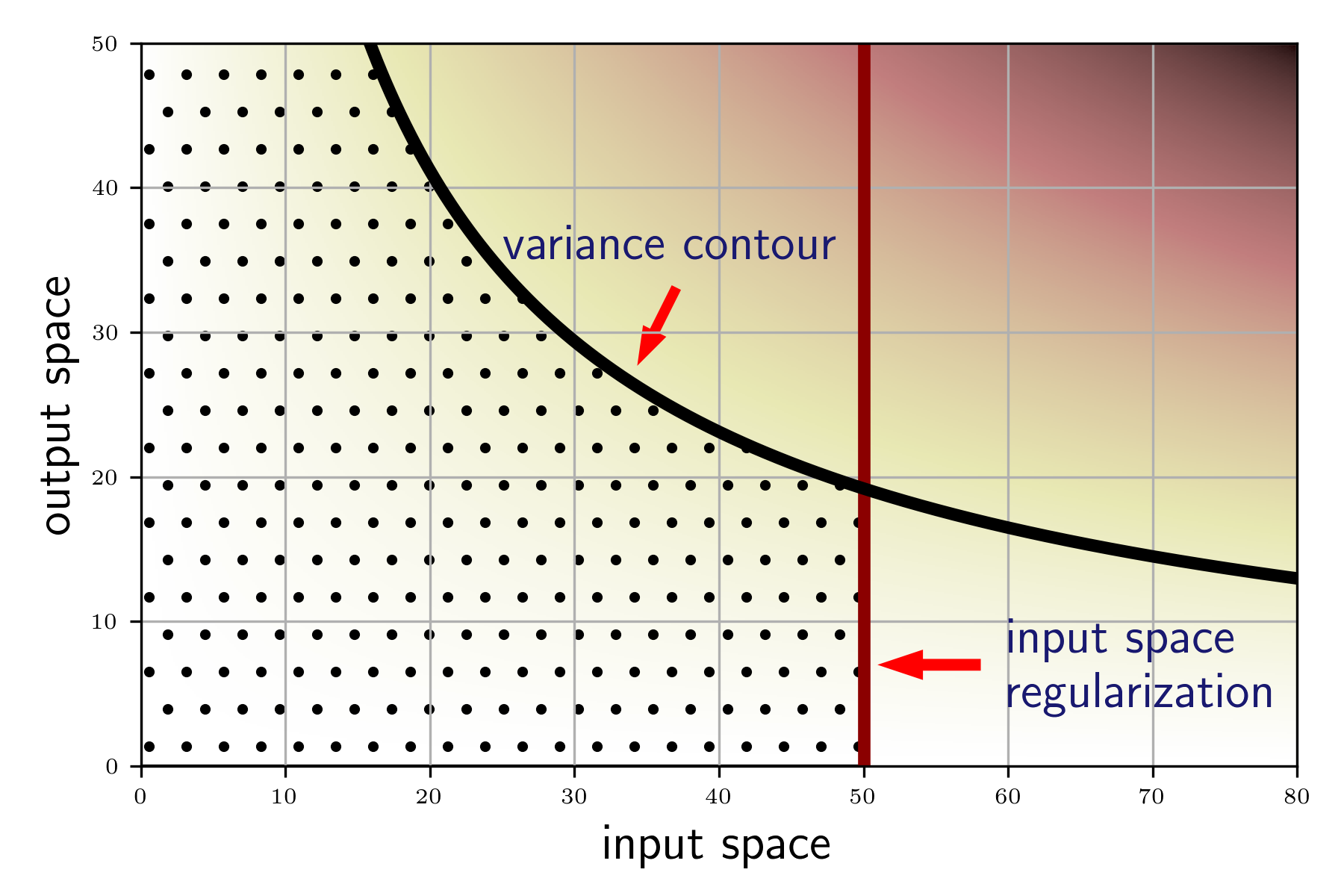}
    \end{subfigure}
    \caption[]%
    {An illustration of our proposed regularization scheme. Left: the regularized least-squares estimator studied in previous works \citep{fischer2020sobolev,de2021convergence,talwai2022sobolev} which only regularizes on the input space. Right: our double regularization scheme via variance contour can achieve the optimal convergence rate in our setting.}
    \label{reg_fig}
\end{figure*}

Since the variance term measures the hardness of learning, we naturally introduce the notion of \textit{variance contour}, which is a curve on the $\mathbb{R}_{+}^2$ plane on which all points induce the same order of variance (here we work with real coordinates for convenience, although we only care about integer points). Formally, we fix an arbitrary constant $C > 0$ and define
\begin{equation}
    \label{variance-contour}
    \ell_{C,\mathtt{var}} = \left\{ (x,y)\in\mathbb{R}_+^2: x^{\frac{\beta'+\max\{\alpha-\beta,p\}}{p}}y^{\frac{1-\gamma'}{q}} = C\right\}.
\end{equation}

A reasonable regularization scheme is then to learn all coordinates $(i,j)\in\mathbb{Z}_+^2$ \textit{below} the curve $\ell_{C,\mathtt{var}}$ and `regularize out' the remaining coordinates that are difficult to learn due to large variances. This can gives us the estimator with smallest estimator at give variance level. 
This observation motivates us to construct our estimator as
\begin{equation}
    \label{our-estimator}
    \hat{\A} = \sum_{j=1}^{y_N} \left( \rho_j^{\half}f_j\otimes \rho_j^{\half}f_j\right) \hat{\C}_{LK}\left(\hat{\C}_{KK}+\lambda_j I\right)^{-1},
\end{equation}
where {$\hat C_{LK} = \frac{1}{N}\sum_{i=1}^N v_i  \otimes u_i$,} $\lambda_j (1\leq j\leq y_N= C^{\frac{q}{1-\gamma'}})$ are the regularization coefficients imposed on different dimensions of the output space. According to \Cref{variance-contour} and noting that $\mu_i \propto i^{-\frac{1}{p}}$, we define
\begin{equation}
    \label{variance-contour-coeff}
    \lambda_j = \max\left\{ \left(j^{-\frac{1-\gamma'}{q}}N^{\max\left\{ 1-\frac{\beta-\beta'}{\max\{\alpha,\beta+p\}},\frac{1-\gamma'}{1-\gamma} \right\}}\right)^{-\frac{1}{\beta'+p}}, c_0\left(\frac{N}{\log N}\right)^{-\frac{1}{\alpha}}  \right\},
\end{equation}
with $C = N^{\max\left\{ 1-\frac{\beta-\beta'}{\max\{\alpha,\beta+p\}},\frac{1-\gamma'}{1-\gamma} \right\}}$ in \Cref{variance-contour}. The additional $N^{-\frac{1}{\alpha}}$ term in \Cref{variance-contour-coeff} is needed for controlling the error of approximating $\C_{KK}$ via $\hat{\C}_{KK}$ (cf. \Cref{conc-covariance}) which is standard in the Sobolev learning literature \citet{fischer2020sobolev,talwai2022sobolev,lu2022sobolev}. The following theorem describes the convergence rate of our estimator defined by \Cref{our-estimator,variance-contour-coeff}. 
\begin{theorem}
\label{thm-upper}
    Consider the estimator $\hat{\A}$ defined by \Cref{our-estimator,variance-contour-coeff}. Suppose that \Cref{asmp-eigendecay,asmp-bounded-input-data,asmp-bounded-output-data,asmp-noise,asmp-beta-gamma-norm} hold, then there exists a universal constant $C$ such that with probability $\geq 1-e^{-\tau}$, we have
    \begin{equation}
        \notag
        \normx{\hat{\A}-\A_0}_{\beta',\gamma'}^2 \leq C \tau^2 \left(\frac{N}{\log N}\right)^{-\min\left\{\frac{\beta-\beta'}{\max\{\alpha,\beta+p\}},\frac{\gamma'-\gamma}{1-\gamma}\right\}} \log^2 N.
    \end{equation}
\end{theorem}

\begin{remark}
    Compared with \Cref{thm:lowerbound}, our upper bound is optimal up to logarithmic factors when $\alpha\leq\beta$. The optimal learning rate in the $\alpha > \beta$ regime is an outstanding problem for decades, even without the additional problem-dependent parameters $\gamma,\gamma'$ (see \emph{e.g.} the discussions following \cite[Theorem 2]{fischer2020sobolev}). In this paper, we do not address this problem either.
\end{remark}

\begin{figure*}
    \centering
    \begin{subfigure}[a]{0.47\textwidth}
        \centering
        \includegraphics[width=\textwidth]{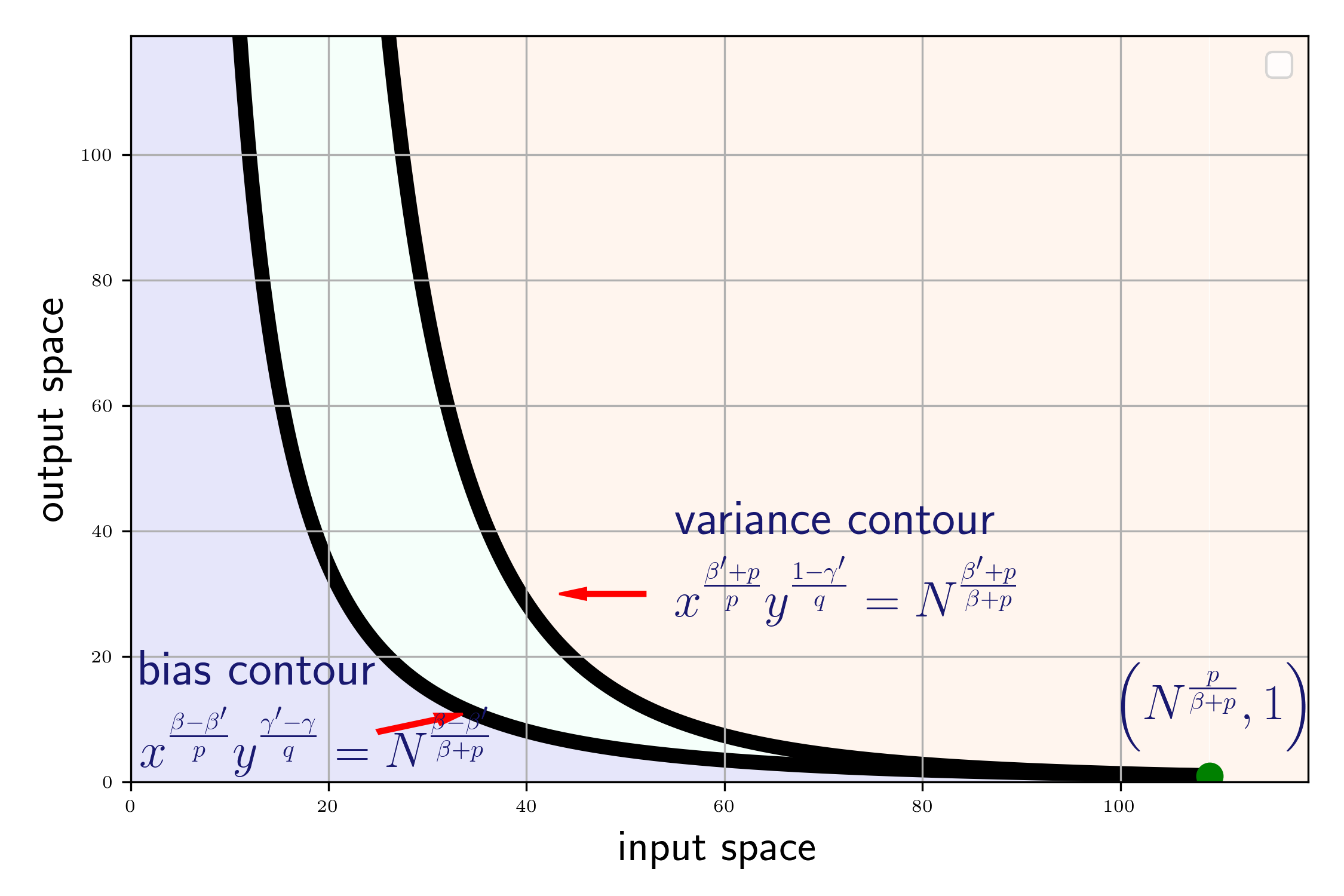}
    \end{subfigure}
    \qquad
    \begin{subfigure}[a]{0.47\textwidth}  
        \centering 
        \includegraphics[width=\textwidth]{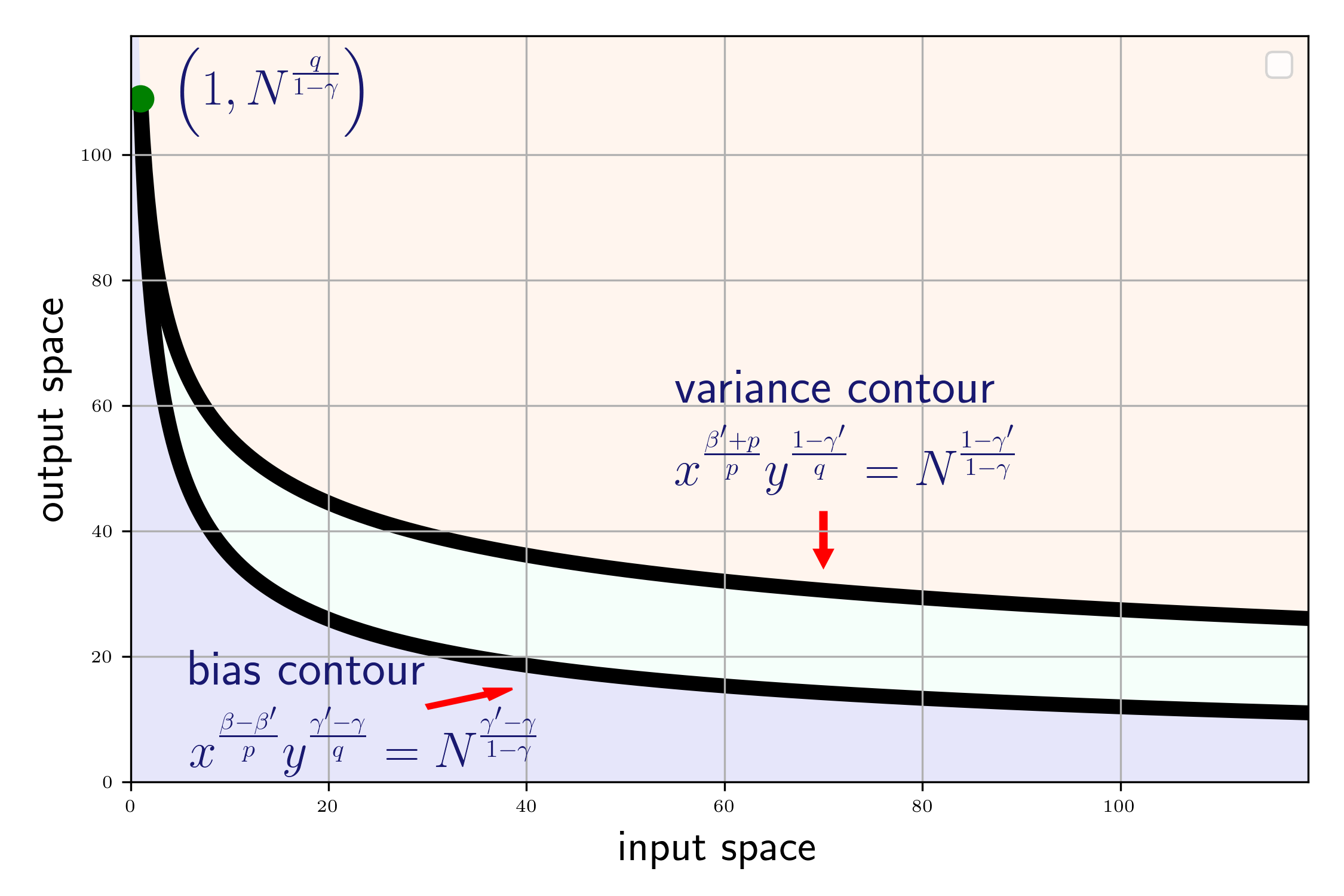}
    \end{subfigure}
    \caption[]%
    {The plot of the bias contour and the variance contour. For simplicity, we only plot the case $\alpha\leq\beta+p$ here. The variance contour is always above the bias contour. Left: When $\frac{\beta'+p}{\beta+p}\geq\frac{1-\gamma'}{1-\gamma}$, the two yields $\O\left(N^{-\frac{\beta-\beta'}{\max\{\alpha,\beta+p\}}}\right)$ convergence rate. It is the same learning rate as the two kernel regression curves meet when $y=1$. Right: When $\frac{\beta'+p}{\beta+p}\geq\frac{1-\gamma'}{1-\gamma}$, the two contours yield the same regularization on the output space leading to a convergence rate of $\O\left(N^{-\frac{\gamma'-\gamma}{1-\gamma}}\right)$}
    \label{contours_fig}
\end{figure*}

\subsection{Regularization via Bias Contour}
\label{subsection:biascon}
We have showed that if we learn all the spectral components under certain variance contour and regularize all other component can achieve optimal rate. In this section, we introduce another 
scheme to design the optimal estimator via learning all the spectral component under a certain bias contour.
Specifically, we consider deciding the regularization strength according to the spectral elements induce a certain level 
of bias i.e. the \textit{bias contour} $\ell_{C',\mathtt{bias}} = \left\{ (x,y)\in\mathbb{R}_+^2: x^{\frac{\beta-\beta'}{p}}y^{\frac{\gamma'-\gamma}{q}}= C'\right\}$.
does not coincide with $\ell_{C,\mathtt{var}}$ for any $C'$ up to constant scaling. Thus, there exists a point $(x^*,y^*)$ on the variance contour with maximal contribution to bias. Naturally, we can also construct our estimator using a bias contour that passes through $(x^*,y^*)$. In this case, we may define $\lambda_j = \max\left\{ \left(j^{-\frac{\gamma'-\gamma}{q}}N^{\min\left\{ \frac{\beta-\beta'}{\max\{\alpha,\beta+p\}},\frac{\gamma'-\gamma}{1-\gamma} \right\}}\right)^{-\frac{1}{\beta-\beta'}}, c_0\left(\frac{N}{\log N}\right)^{-\frac{1}{\alpha}}  \right\}$ for similar reasons as \Cref{subsec-main-idea}, which also yields optimal rate as stated in \Cref{thm-upper-bias} below.

\begin{remark}[On the optimal shape of regularization]
\label{remark4.1} The discussion in \Cref{subsec-main-idea,subsection:biascon} reveals another understanding of our information theoretic lower bound. Firstly, we should learn all the spectral components under the bias contour otherwise the bias will exceed the lower bound. Secondly, we should not learn any spectral component over the variance contour since otherwise the variance will exceed the lower bound. Thus \alertinline{the bias contour should always be under the variance contour}, otherwise no estimator can be designed. The bias and variance contours at the level of optimal learning rate are plotted in \Cref{contours_fig}. They only meet at $(x^*,y^*)$ with $x^*=1$ or $y^*=1$, which has the largest contribution to the bias (resp. variance) among all points on the variance (resp. bias) contour, thus dominating the estimation error. When the two curves meet at $y^*=1$, it reduces to the original kernel regression case. When the two curves meet at $x^*=1$, it leads to our new rate that depends on the output space.

\end{remark}

\begin{theorem}
    \label{thm-upper-bias}
    Consider the estimator $\hat{\A}$ defined by \Cref{our-estimator} with $\lambda_j$ defined above. Suppose that \Cref{asmp-eigendecay,asmp-bounded-input-data,asmp-bounded-output-data,asmp-noise,asmp-beta-gamma-norm} hold, then there exists a universal constant $C$, such that 
    \begin{equation}
        \notag
        \normx{\hat{\A}-\A_0}_{\beta',\gamma'}^2 \leq C \tau^2 \left(\frac{N}{\log N}\right)^{-\min\left\{\frac{\beta-\beta'}{\max\{\alpha,\beta+p\}},\frac{\gamma'-\gamma}{1-\gamma}\right\}} \log^2 N
    \end{equation}
    holds with probability $\geq 1-e^{-\tau}$.
\end{theorem}

\section{MultiLevel Kernel Operator Learning}
\label{sec-multilevel}

In this section, we study a multilevel machine learning algorithm \citep{lye2021multi,li2021multilevel,boulle2022learning} but at each level we consider a cost-accuracy trade-off \citep{de2022cost} to control the variance at a proper scale. We show that the multilevel level algorithm can cover all the spectral component below the bias contour and achieve the optimal learning rate.  Our idea is similar to the multilevel Monte Carlo \citep{giles2008multilevel,giles2015multilevel}, which reduces bias from multilevel algorithm.  Our multilevel estimator differs from the DeepONet \citep{lu2019deeponet} and the PCA-Net \citep{bhattacharya2020model} since we add different regularizations for each level. Our theory indicates that the multilevel approach outperforms previous ones and achieves the optimal learning rate.



\begin{figure*}
    \centering
    \begin{subfigure}[a]{0.47\textwidth}
        \centering
        \includegraphics[width=\textwidth]{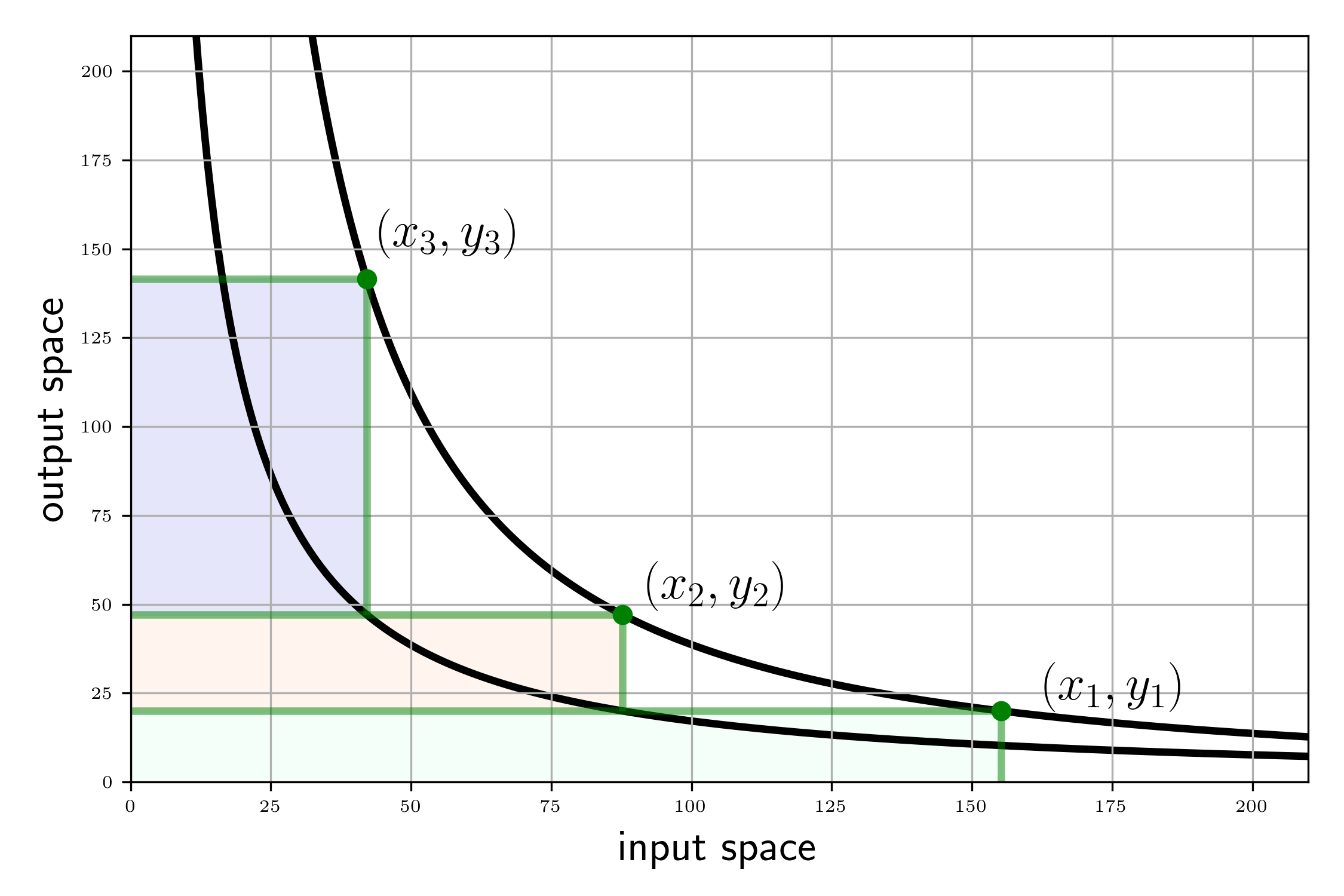}
    \end{subfigure}
    \qquad
    \begin{subfigure}[a]{0.47\textwidth}  
        \centering 
        \includegraphics[width=\textwidth]{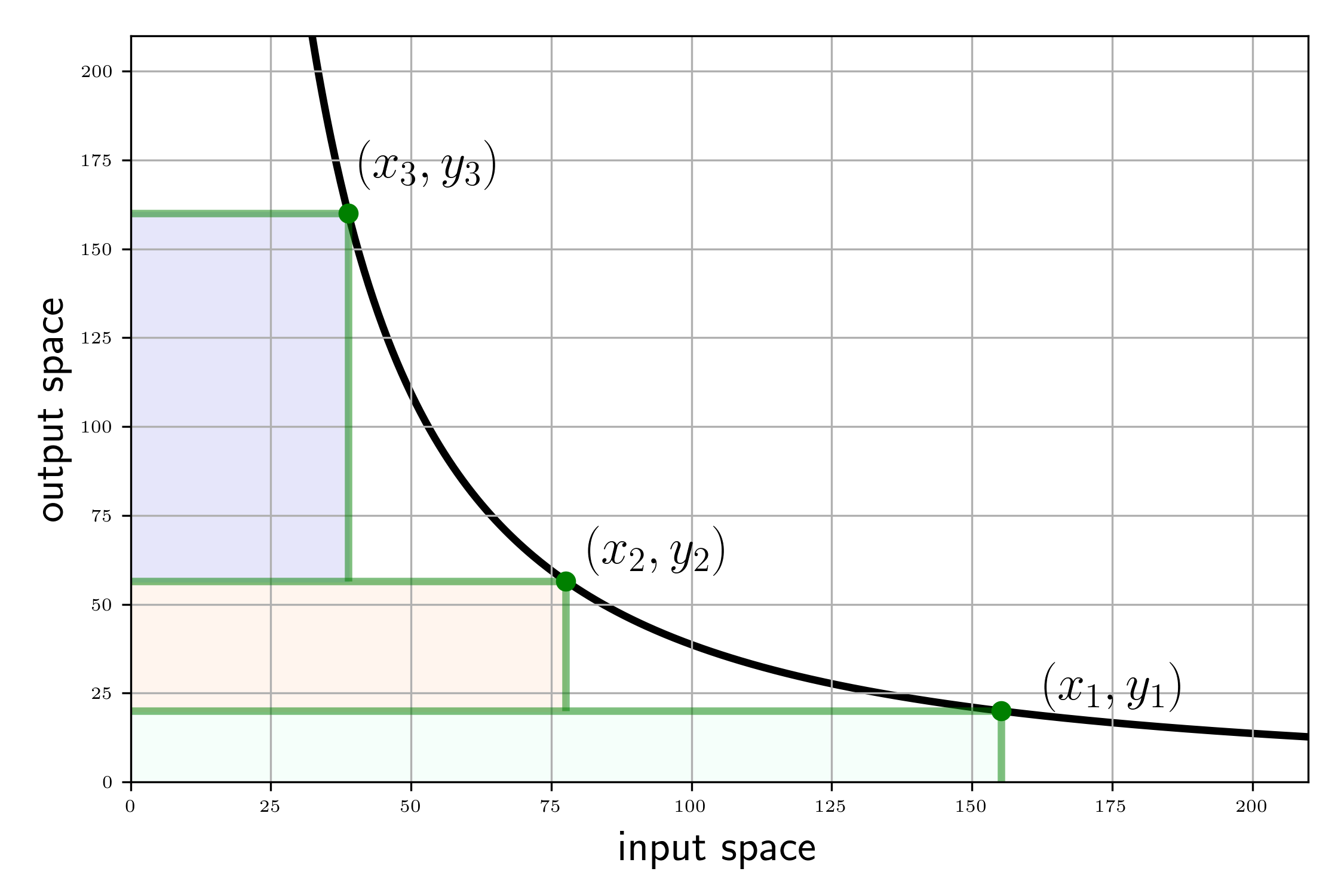}
    \end{subfigure}
    \vspace{0.05in}
    \caption[]%
    {Construction of the sequence $\{(x_i,y_i)\}$. Left: the case $\frac{\beta-\beta^{\prime}}{\max \{\alpha, \beta+p\}} \neq \frac{\gamma^{\prime}-\gamma}{1-\gamma}$. Right: the case $\frac{\beta-\beta^{\prime}}{\max \{\alpha, \beta+p\}}=\frac{\gamma^{\prime}-\gamma}{1-\gamma}$, where the bias and variance contours overlap and we set $x_{n+1}=\half x_n$. Each rectangular represents a certain level of regularization.}
    \label{multilevel_fig}
\end{figure*}

The basic idea is to design a minimum number of machine learning estimators that cover all the spectral elements under the bias contour but do not exceed the variance contour at the same time.  To achieve this, we choose sequences $\{x_i\}$ and $\{y_i\}$ for $1\leq i\leq L_N$ where $y_i$ denotes the $i$-th level and $x_i$ controls the corresponding regularization via the regularization coefficient $\lambda_i^{(K)}=x_i^{-\frac{1}{p}}$. The sequences are chosen in a staircase manner as plotted in \Cref{multilevel_fig} (for formal definitions see \Cref{appsec:multilevel}). The eigenbasis $\left\{\rho_j^{\half}f_j\right\}$ of the output space is divided into defferent levels by $\{y_i\}$. The main idea behind our multilevel method is that different levels of the output need to be learned with different regularization. Formally, we define our multilevel estimator as
\begin{equation}
    \label{multilevel-estimator}
    \hat{\A}_{\mathtt{ml}} =  \sum_{i=0}^{L_N}\left(\sum_{y_{i-1}\leq j < y_{i}}\rho_j^{\half}f_j\otimes\rho_j^{\half}f_j\right)\hat{\C}_{LK}\left(\hat{\C}_{KK}+\lambda_i^{(K)}I\right)^{-1}.
\end{equation}


The following theorem shows that the estimator \Cref{multilevel-estimator} can achieve the optimal convergence rate with $L_N=\O(\ln\ln N)$ when $\frac{\beta-\beta^{\prime}}{\max \{\alpha, \beta+p\}} \neq \frac{\gamma^{\prime}-\gamma}{1-\gamma}$. We also show that $O(\ln N)$ estimator is needed for the case when $\frac{\beta-\beta'}{\max\{\alpha,\beta+p\}}=\frac{\gamma'-\gamma}{1-\gamma}$ (\cref{multilevel_fig} Right) in \Cref{appsec:multilevel}.
\begin{theorem}
    \label{thm-upper-multilavel}
    Suppose that \Cref{asmp-eigendecay,asmp-bounded-input-data,asmp-bounded-output-data,asmp-noise,asmp-beta-gamma-norm} hold, then there exists a sequence $\{y_i\}_{1\leq i\leq L_N}$ with $L_N = \O(\ln N)$ when $\frac{\beta-\beta'}{\max\{\alpha,\beta+p\}}=\frac{\gamma'-\gamma}{1-\gamma}$ and $\O(\ln\ln N)$ otherwise, such that the estimator $\hat{\A}_{\mathtt{ml}}$ satisfies $\normx{\hat{\A}_{\mathtt{ml}}-\A_0}_{\beta',\gamma'}^2 \leq C \tau^2 \left(\frac{N}{\log N}\right)^{-\min\left\{\frac{\beta-\beta'}{\max\{\alpha,\beta+p\}},\frac{\gamma'-\gamma}{1-\gamma}\right\}} \log^2 N$ with probability $\geq 1-e^{-\tau}$, where $C$ is a universal constant.
\end{theorem}

\begin{remark}
{ Our multilevel algorithm first apply the regression algorithm on low-frequency projections of the output samples with small regularization and then successively fine-tune the regression model on high-frequency projections of the output samples with stronger regularization, which matches the empirical use \citep{li2021multilevel,lye2021multi}.}
\end{remark}


\section{Conclusion and Discussion}
We considered the sample complexity of learning an operator between two infinite-dimensional Sobolev kernel Hilbert spaces. We provided an information theoretical lower bound for this problem along with a multi-level machine learning algorithm. Our lower bound is determined by the harder rate of two polynomial rates: one is fully determined by the hardness of input space while the other is fully controlled by the hardness of the output space. The second rate is new to the literature. We explained the our bound from the viewpoint of variance and bias counters in \Cref{remark4.1} and \Cref{contours_fig}. The optimal estimator should learn all the spectral element under the bias contour but learn no information above the variance contour. To meet this requirement, we combined the idea of multilevel Monte Carlo with kernel operator learning, using successive levels to fit higher frequency information while keeping the variance at the same scale in order to reduce the bias. Our paper is the first paper on the non-parametric statistical optimality for multilevel algorithms. We leave estimation from discretely observe functional covariates with noise as future work  \citep{zhoufunctional,zhou2022theory}.

\begin{acknowledgement}
 Jikai Jin is partially supported by the elite undergraduate training program of School of
Mathematical Sciences in Peking University. Yiping Lu is supported by the Stanford Interdisciplinary Graduate
Fellowship (SIGF). Jose Blanchet is supported in part by the Air Force Office of Scientific Research under award number FA9550-20-1-0397.  Lexing Ying is supported is supported by National Science Foundation under award DMS-2208163.
\end{acknowledgement}


\printbibliography

\appendix

\section{Proof of the lower bound}
\label{appsec:lower-proof}
In this section, we follow the lower bound proof in \citep{fischer2020sobolev} to give a lower bound of the convergence rate in our operator learning setting.

\subsection{Preliminaries on Tools for Lower Bounds}
\label{appendix:fano}
In this section, we repeat the standard tools we use to establish the lower bound. The main tool we use is the Fano's inequality and the Varshamov-Gilber Lemma.

\begin{lemma}[Fano's methods] Assume that $V$ is a uniform random variable over set $\mathcal{V}$, then for any Markov chain $V\rightarrow X\rightarrow \hat V$, we always have
$$
\mathcal{P}(\hat V\not = V)\ge 1-\frac{I(V;X)+\log 2}{\log(|\mathcal{V}|)}
$$

\end{lemma}

In our proof we will use a version from \citep{fischer2020sobolev}.

\begin{lemma}
\label{lower-bound-lemma}
~\cite[Theorem 20]{fischer2020sobolev} Let $M \geq 2,(\Omega, \mathcal{A})$ be a measurable space, $P_{0}, P_{1}, \ldots, P_{M}$ be probability measures on $(\Omega, \mathcal{A})$ with $P_{j} \ll P_{0}$ for all $j=1, \ldots, M$, and $0<\alpha_{*}<\infty$ with
\begin{equation}
    \notag
    \frac{1}{M} \sum_{j=1}^{M} KL\left(P_{j}|| P_{0}\right) \leq \alpha_{*}.
\end{equation}
Then, for all measurable functions $\Psi: \Omega \rightarrow\{0,1, \ldots, M\}$, the following bound is satisfied
\begin{equation}
    \notag
    \max _{j=0,1, \ldots, M} P_{j}(\omega \in \Omega: \Psi(\omega) \neq j) \geq \frac{\sqrt{M}}{1+\sqrt{M}}\left(1-\frac{3 \alpha_{*}}{\log (M)}-\frac{1}{2 \log (M)}\right) .
\end{equation}
\end{lemma}

\begin{lemma}[Varshamov-Gillbert Lemma,\citep{tsybakov2008introduction} Theorem 2.9] Let $D\ge 8$. There exists a subset $\mathcal{V}=\{\tau^{(0)},\cdots,\tau^{(2^{D/8})}\}$ of $D-$dimensional hypercube $\mathcal{H}^D=\{0,1\}^D$ such that $\tau^{(0)}=(0,0,\cdots,0)$ and the $\ell_1$ distance between every two elements is larger than $\frac{D}{8}$
$$
\sum_{l=1}^D \normx{\tau^{(j)}-\tau^{(k)}}_{\ell_1}\ge \frac{D}{8}\text{, for all }0\le j,k \le 2^{D/8}
$$
\end{lemma}
\subsection{Proof of the Lower Bound}
To prove our lower bound, we construct a sequence of linear operators as follows:
\begin{equation}
    \notag
    \A_{\omega} = \sqrt{\frac{32\eps}{m_1 K}}\sum_{i=1}^{m_1}\sum_{j=1}^K\omega_{ij}\mu_{i+m_1}^{\beta'/2}\rho_{j+m_2}^{1-\gamma'/2} f_{j+m_2}\otimes e_{i+m_1},\quad \omega_{ij} \in \{0,1\}
\end{equation}
where $m_1$ and $m_2$ are hyper-parameters (scale as $\mathrm{poly}(N)$ and will be selected later) and $K$ is a constant that will be specified afterwards. It's easy to check that
\begin{equation}
    \notag
     \normx{\A_{\omega}-\A_{\omega'}}_{\beta',\gamma'}^2 = \frac{32\eps}{m_1 K}\sum_{i=1}^{m_1}\sum_{j=1}^K\left(\omega_{ij}-\omega_{ij}'\right)^2
\end{equation}
By Gilbert-Varshamov Lemma it is possible to select $M_{\eps} \geq 2^{m_1 K/8}$ binary strings
\begin{equation}
    \notag
    \omega^{(1)},\omega^{(2)},\cdots,\omega^{(M_{\eps})} \in \{0,1\}^{m_1 K}
\end{equation}
such that $\normx{\omega^{(i)}-\omega^{(j)}}_2^2 \geq 4\eps$. Let $\Omega$ be the collection of this strings. 

We now select the hyper-parameters to satisfies the assumptions made in \Cref{sec-problem}. First we have
\begin{equation}
    \notag
    \normx{\A_{\omega}}_{\beta,\gamma}^2 \leq \frac{32\eps}{m_1 K}\sum_{i=1}^{m_1}\sum_{j=1}^{K}\mu_{i+m_1}^{-(\beta-\beta')}\rho_{j+m_2}^{-(\gamma'-\gamma)} \lesssim \eps \left(2m_1\right)^{\frac{\beta-\beta'}{p}}\left(2m_2\right)^{\frac{\gamma'-\gamma}{q}}
\end{equation}
where the last step follows from \Cref{asmp-eigendecay}. Similarly, we have $\normx{\A_{\omega}}_{\alpha,1}^2 \lesssim \eps \left(2m_1\right)^{\frac{\alpha-\beta'}{p}}\left(2m_2\right)^{\frac{\gamma'-1}{q}}.$ To the assumptions made in \Cref{sec-problem}, we should make
\begin{equation}
    \label{rectangle_ineq}
    (2m_1)^{\frac{\max\{\alpha,\beta\}-\beta'}{p}}(2m_2)^{\frac{\gamma'-\gamma}{q}} \lesssim \eps^{-1}
\end{equation}
be satisfied. To be specific, with the previous selection of hyper-parameters,  we can  have $\normx{\A_{\omega}}_{\beta,\gamma} = \O(1)$ and
\begin{equation}
    \notag
    \sup_{g\in \mathrm{range}(\A_{\omega})}\normx{g}_{\H_L} \leq \sup_{f}\normx{ \A_{\omega}}_{\alpha,1}\cdot\normx{\left(I_{1,\alpha,P_K}^*\right)^{\dagger}f}_{\H_K^{\alpha}} < +\infty
\end{equation}
where the last step follows from our assumption on the input distribution \Cref{asmp-bounded-input-data}. This verifies that \Cref{asmp-beta-gamma-norm,asmp-bounded-output-data} hold for $\A_{\omega},\forall \omega\in\Omega$.

We now construct the hypothesis (probability distributions) as follows: for $\forall \omega \in \{0,1\}^{m_1}$, define
\begin{equation}
    \notag
    P_{\omega}(\mathrm{d} u, \mathrm{d} v) = \mathrm{d} \mathcal{N}\left( \A_{\omega} u, \Sigma\right)(v)\cdot \mathrm{d} P_K(u)
\end{equation}
where the covariance operator $\Sigma = \frac{\sigma^2}{K} \sum_{j=1}^K\rho_{j+m_2} u_{j+m_2}\otimes u_{j+m_2}$ for some constant $\sigma>0$. It's then easy to see that $\tr{\Sigma}=\sigma^2$, which satisfies \Cref{asmp-noise} .Note that the range of $C_{\omega}$ is  $\mathrm{span}(u_{m_2})$ and $\Sigma$ is non-degenerate on this subspace. As a result, we can view $P_{\omega},\omega\in \Omega$ as distributions on $\H_K \times \mathrm{span}\left(u_{j+m_2}:1\leq j\leq K\right)$, and we have for $\forall \omega,\omega' \in \Omega$ that
\begin{equation}
    \notag
    \begin{aligned}
    KL\left(P_{\omega}||P_{\omega'}\right) 
    &= \E_{u\sim P_K}\left[ KL\left( P_{\omega}(\mathrm{d} v\mid u)||  P_{\omega'}(\mathrm{d} v\mid u)\right)\right] \\
    &= \E_{f\sim P_K}\left[ KL\left( \mathcal{N}(\A_{\omega}u,\Sigma) || \mathcal{N}(\A_{\omega'}u,\Sigma)\right) \right] \\
    &= \E_{u\sim P_K} \left\langle (\A_{\omega}-\A_{\omega'})u, \Sigma^{\dagger} (\A_{\omega}-\A_{\omega'})u\right\rangle \\
    &\leq \sigma^{-2} K \E_{u\sim P_K} \left\langle (\A_{\omega}-\A_{\omega'})u, (\A_{\omega}-\A_{\omega'})u\right\rangle \\
    &= \frac{32\eps}{m_1\sigma^2} \E_{u\sim P_K}  \normx{\sum_{i=1}^{m_1}\sum_{j=1}^K(\omega_{ij}-\omega_{ij}')\mu_{i+m_1}^{\beta'/2}\rho_{j+m_2}^{1-\gamma'/2}\left\langle u,e_{i+m_1}\right\rangle f_{j+m_2}}_{\H_L}^2 \\
    &= \frac{32\eps}{m_1\sigma^2}\E_{u\sim P_K} \sum_{j=1}^K\rho_{j+m_2}^{1-\gamma'}\left( \sum_{i=1}^{m_1}(\omega_{ij}-\omega_{ij}')\mu_{i+m_1}^{\beta'/2}\left\langle u,e_{i+m_1}\right\rangle\right)^2 \\
    &= \frac{32\eps}{m_1\sigma^2}\sum_{i=1}^{m_1}\sum_{j=1}^K(\omega_{ij}-\omega_{ij}')^2 \mu_{i+m_1}^{\beta'}\rho_{j+m_2}^{1-\gamma'} \lesssim \eps\sigma^{-2}m_1^{-\frac{\beta'}{p}}m_2^{-\frac{1-\gamma'}{q}}
    \end{aligned}
\end{equation}
where the last step follows from $\E_{P_K}u\otimes u = \mathcal{C}_{KK} = \sum_{i=1}^{\infty}\mu_i^2 e_i\otimes e_i$ and recall that $K$ is a constant. Hence we deduce that
\begin{equation}
    \notag
    \frac{1}{M_{\eps}}\sum_{\omega'\in\Omega} KL(P_{\omega'}^N||P_{\omega}^N) \lesssim \sigma^{-2}N\eps m_1^{-\frac{\beta'}{p}}m_2^{-\frac{1-\gamma'}{q}} =: \alpha^*
\end{equation}
Applying \Cref{lower-bound-lemma}, we find that when
\begin{equation}
    \notag
    \alpha^* \lesssim \log M_{\eps} \Leftrightarrow \eps \lesssim N^{-1}m_1^{\frac{\beta'}{p}}m_2^{\frac{1-\gamma'}{q}},
\end{equation}
there exists a hypothesis $P_{\omega_0}$ such that for any estimator $\hat{\A}_{\omega_0}$,
\begin{equation}
    \notag
    \left\{ ||\hat{\A}_{\omega_0}-\A_{\omega_0}||_{\beta',\gamma'}^2 \geq \eps\right\} \supset \left\{ \omega_0 \neq \argmin_{\omega\in\Omega} || \hat{\A}_{\omega}-\A_{\omega_0}||_{\beta',\gamma'}\right\}
\end{equation}
holds with high probability.

Finally, we need to choose optimal $m_1$ and $m_2$ under the constraint \Cref{rectangle_ineq}. It turns out that either $m_1=1$ or $m_2=1$, and the resulting lower bound is
\begin{equation}
    \notag
    \normx{\hat{\A}-\A_0}_{\beta',\gamma'} \gtrsim N^{-\min\left\{ \frac{\max\{\alpha,\beta\}-\beta'}{2\left(\max\{\alpha,\beta\}+p\right)},\frac{\gamma'-\gamma}{2\left(1-\gamma\right)} \right\}}.
\end{equation}

\section{Proof of the upper bound}
\label{appsec:upper-proof}
In this section, we upper-bound the learning error of estimator \Cref{our-estimator} which defined as
\begin{equation}
    \hat{\A} = \sum_{j=1}^{y_N} \left( \rho_j^{\half}f_j\otimes \rho_j^{\half}f_j\right) \hat{\C}_{LK}\left(\hat{\C}_{KK}+\lambda_j I\right)^{-1},
\end{equation}
where $\lambda_j,1\leq j\leq y_N=N^{\frac{q}{1-\gamma'}\max\{1-\frac{\beta-\beta'}{\max\{\alpha,\beta+p\}},\frac{1-\gamma'}{1-\gamma}\}}$ are regularization coefficients that we impose on different dimensions of the output space. In this section, we consider the following two ways to select regularization coefficients in \Cref{sec-estimator}:
\begin{itemize}
    \item We regularize all spectral component below certain variance contour, \emph{i.e.} we set regularization strength {\small $\lambda_j = \max\left\{ \left(j^{-\frac{1-\gamma'}{q}}N^{\max\left\{ 1-\frac{\beta-\beta'}{\max\{\alpha,\beta+p\}},\frac{1-\gamma'}{1-\gamma} \right\}}\right)^{-\frac{1}{\beta'+p}}, c_0\left(\frac{N}{\log N}\right)^{-\frac{1}{\alpha}}  \right\}$} \Cref{variance-contour-coeff}.
    \item We regularize all spectral component below certain bias contour, \emph{i.e.} we set regularization strength {\small $\lambda_j = \max\left\{ \left(j^{-\frac{\gamma'-\gamma}{q}}N^{\min\left\{ \frac{\beta-\beta'}{\max\{\alpha,\beta+p\}},\frac{\gamma'-\gamma}{1-\gamma} \right\}}\right)^{-\frac{1}{\beta-\beta'}}, c_0\left(\frac{N}{\log N}\right)^{-\frac{1}{\alpha}}  \right\}$} \Cref{bias-contour-coeff-app}.
\end{itemize}

To obtain the upper bound for our estimator, we decompose the learning error $\mathcal{E}(\hat{\A}) = \normx{\hat{\A}-\A_0}_{\beta',\gamma'}$ in to bias and variance via
\begin{equation}
    \notag
    \mathcal{E}(\A) \leq \underbrace{|| \hat{\A}-\A_{\lambda}||_{\beta',\gamma'}}_{\text{variance term}} + \underbrace{|| \A_{\lambda} - \A_0||_{\beta',\gamma'}}_{\text{bias term}},
\end{equation}
where
\begin{equation}
    \label{regularized-op}
    \A_{\lambda} = \sum_{j=1}^{y_N} \left( \rho_j^{\frac{1}{2}}f_j \otimes \rho_j^{\frac{1}{2}} f_j\right) \mathcal{C}_{KL} \left(\mathcal{C}_{KK}+\lambda_j I\right)^{-1}.
\end{equation}

\subsection{Regularization via Variance Counter}

In the following, we separately bound the bias term and the variance term. We first assume $\alpha\leq\beta+p$ in \Cref{sec-bias} and \Cref{sec-var}, then the case $\alpha>\beta+p$ is treated in \Cref{sec-hard}. Finally in \Cref{bias-contour-rate}, we establish the same convergence rate for regularization via bias contour.

\subsubsection{Bias}
\label{sec-bias}

\begin{lemma}

$||\A_0-\A_{\lambda}||_{\beta',\gamma'}^2\lesssim N^{-\min\left\{\frac{\beta-\beta'}{\beta+p},\frac{\gamma'-\gamma}{1-\gamma}\right\}}.$

\end{lemma}

\begin{proof-sketch} Since $\normx{\A_0}_{\beta,\gamma}\leq B$, we can write $\A_0 := \sum_{i=1}^{+\infty}\sum_{j=1}^{+\infty}a_{ij}\mu_i^{\frac{\beta}{2}}\rho_j^{1-\frac{\gamma}{2}} f_j \otimes e_i$ where the coefficient matrix $A_0 = (a_{ij})_{1\leq i,j\leq +\infty}$ satisfies $\normx{A_0}_{F}^2 \leq B^2$. The definition \Cref{regularized-op} implies that for $1\leq j\leq y_N$ and $i\geq 1$ we have
\begin{equation}
    \notag
    \begin{aligned}
        \left\langle \rho_j^{\half}f_j, \A_{\lambda}\mu_i^{\half}e_i\right\rangle &=
        \left\langle \rho_j^{\half}f_j, \mathcal{C}_{KL}\left(\mathcal{C}_{KK}+\lambda_j I\right)^{-1}\mu_i^{\half}e_i\right\rangle \\
        &= \left\langle \rho_j^{\half}f_j, \A_0 \C_{KK}\left(\mathcal{C}_{KK}+\lambda_j I\right)^{-1}\mu_i^{\half}e_i\right\rangle 
        = \frac{\mu_i^{\frac{1+\beta}{2}}}{\mu_i+\lambda_j}\rho_j^{\frac{1-\gamma}{2}} a_{ij}.
    \end{aligned}
\end{equation}
The bias term can be bounded as follows:
\begin{equation}
    \label{bias-ineq}
    \begin{aligned}
        ||\A_0-\A_{\lambda}||_{\beta',\gamma'}^2
        &= \sum_{i,j=1}^{+\infty} \left\langle \rho_j^{\frac{1}{2}}f_j, \mathcal{C}_{Q_{K}}^{-\frac{1-\gamma'}{2}}\left(\A_0-\A_{\lambda}\right)\mathcal{C}_{KK}^{\frac{1-\beta'}{2}} \mu_i^{\frac{1}{2}}e_i\right\rangle^2\\
        &= \sum_{j=1}^{y_N}\sum_{i=1}^{+\infty} \mu_i^{\beta-\beta'}\rho_j^{\gamma'-\gamma} \frac{\lambda_j^2}{(\mu_i+\lambda_j)^2}a_{ij}^2 \\
        &\leq  \sum_{j=1}^{y_N}\rho_j^{\gamma'-\gamma}\max_{i\geq 1}\left( \mu_i^{\beta-\beta'}\frac{\lambda_j^2}{(\mu_i+\lambda_j)^2}\right)\cdot\sum_{i=1}^{+\infty} a_{ij}^2 \\
        &\lesssim \sum_{j=1}^{y_N} j^{-\frac{\gamma'-\gamma}{q}} \lambda_j^{-(\beta-\beta')}\sum_{i=1}^{+\infty} a_{ij}^2 \lesssim B^2 \max_{1\leq j\leq y_N} j^{-\frac{\gamma'-\gamma}{q}} \lambda_j^{-(\beta-\beta')}.
    \end{aligned}
\end{equation}
We now prove that
\begin{equation}
    \label{bias-ineq-main}
    j^{\frac{\gamma'-\gamma}{q}} \lambda_j^{\beta-\beta'} \gtrsim N^{\min\left\{\frac{\beta-\beta'}{\beta+p},\frac{\gamma'-\gamma}{1-\gamma}\right\}},\quad \forall 1\leq j\leq y_N.
\end{equation}

\textbf{Case 1.} If $\lambda_j = c_0\left(\frac{N}{\log N}\right)^{\frac{1}{\alpha}}$, then
\begin{equation}
    \notag
    j^{\frac{\gamma'-\gamma}{q}} \lambda_j^{\beta-\beta'} \geq \lambda_j^{\beta-\beta'} \gtrsim N^{\frac{\beta-\beta'}{\alpha}} \geq N^{\frac{\beta-\beta'}{\beta+p}}
\end{equation}
where we use $\alpha\leq \beta+p$ in the final step.

\textbf{Case 2.} If $\lambda_j = \left( N^{\max\left\{ \frac{\beta'+p}{\beta+p},\frac{1-\gamma'}{1-\gamma} \right\}} j^{-\frac{1-\gamma'}{q}}\right)^{\frac{1}{\beta'+p}}$, we need to consider two sub-cases:
\begin{itemize}
    \item If $\frac{\beta'+p}{\beta+p}>\frac{1-\gamma'}{1-\gamma}$, then we have $\lambda_j = \left( N^{\frac{\beta'+p}{\beta+p}} j^{-\frac{1-\gamma'}{q}}\right)^{\frac{1}{\beta'+p}}$ and thus
    \begin{equation}
        \notag
        j^{\frac{\gamma'-\gamma}{q}} \lambda_j^{\beta-\beta'} = j^{\frac{\gamma'-\gamma}{q}} \left( N^{\frac{\beta'+p}{\beta+p}} j^{-\frac{1-\gamma'}{q}}\right)^{\frac{\beta-\beta'}{\beta'+p}} =
        N^{\frac{\beta-\beta'}{\beta+p}} j^{\frac{1-\gamma'}{q}\left(\frac{\gamma'-\gamma}{1-\gamma'}-\frac{\beta-\beta'}{\beta'+p}\right)} \geq N^{\frac{\beta-\beta'}{\beta+p}}.
    \end{equation}
    \item If $\frac{\beta'+p}{\beta+p}<\frac{1-\gamma'}{1-\gamma}$, then similarly we have $\lambda_j = \left( N^{\frac{1-\gamma'}{1-\gamma}} j^{-\frac{1-\gamma'}{q}}\right)^{\frac{1}{\beta'+p}}$ and
    \begin{equation}
        \notag
        j^{\frac{\gamma'-\gamma}{q}} \lambda_j^{\beta-\beta'} = j^{\frac{\gamma'-\gamma}{q}} \left( N^{\frac{1-\gamma'}{1-\gamma}} j^{-\frac{1-\gamma'}{q}}\right)^{\frac{\beta-\beta'}{\beta'+p}} \geq  y_N^{\frac{\gamma'-\gamma}{q}} \left( N^{\frac{1-\gamma'}{1-\gamma}} y_N^{-\frac{1-\gamma'}{q}}\right)^{\frac{\beta-\beta'}{\beta'+p}} = N^{\frac{\gamma'-\gamma}{1-\gamma}}.
    \end{equation}
\end{itemize}
Hence, in all cases \Cref{bias-ineq-main} holds and we have that
\begin{equation}
    \label{bias-upper-bound}
    \normx{\A_0-\A_{\lambda}}_{\beta',\gamma'}^2\lesssim N^{-\min\left\{\frac{\beta-\beta'}{\beta+p},\frac{\gamma'-\gamma}{1-\gamma}\right\}}.
\end{equation}
\end{proof-sketch}
\subsubsection{Variance}
\label{sec-var}
The variance term can be rewritten in the following way:
\begin{subequations}
    \label{variance-rewrite}
    \begin{align}
    \mathcal{V} &= \normx{\hat{\A}-\A_{\lambda}}_{\beta',\gamma'}^2 =  \normx{\mathcal{C}_{Q_{L}}^{-\frac{1-\gamma'}{2}}\left( \hat{\A}-\A_{\lambda}\right) \mathcal{C}_{KK}^{\frac{1-\beta'}{2}}}_{\HS}^2 \nonumber\\
    &= \sum_{i,j=1}^{+\infty} \left\langle \rho_j^{\frac{1}{2}}f_j, \mathcal{C}_{Q_{L}}^{-\frac{1-\gamma'}{2}}\left( \hat{\A}-\A_{\lambda}\right) \mathcal{C}_{KK}^{\frac{1-\beta'}{2}} \mu_i^{\frac{1}{2}}e_i\right\rangle^2 \label{variance-rewrite-1}\\
    &= \sum_{j=1}^{n_{N}} \rho_j^{-(1-\gamma')}\sum_{i=1}^{+\infty} \left\langle \rho_j^{\frac{1}{2}}f_j,  \left[\hat{\C}_{LK}\left( \hat{\mathcal{C}}_{KK}+\lambda_j I\right)^{-1}-\C_{LK}\left( \C_{KK}+\lambda_j I\right)^{-1}\right]\mu_i^{1-\frac{\beta'}{2}}e_i\right\rangle^2 \label{variance-rewrite-2} \\
    &= \sum_{j=1}^{n_{N}} \rho_j^{-(1-\gamma')}\sum_{i=1}^{+\infty} \left\langle \underbrace{\left(\C_{KK}+\lambda_j I\right)^{-\frac{1}{2}}\left[ \hat{\C}_{KL} - \left(\hat{\mathcal{C}}_{KK}+\lambda_j I\right) \left( \mathcal{C}_{KK}+\lambda_j I\right)^{-1} \C_{KL} \right]}_{=: U_j} \rho_j^{\frac{1}{2}}f_j, \right.\nonumber\\ 
    &\qquad \left.\underbrace{\left( \mathcal{C}_{KK}+\lambda_j I\right)^{\frac{1}{2}}\left( \hat{\mathcal{C}}_{KK}+\lambda_j I\right)^{-1}\left(\mathcal{C}_{KK}+\lambda_j I\right)^{\frac{1}{2}}}_{=: G_j}\frac{\mu_i^{1-\frac{\beta'}{2}}}{\sqrt{\mu_i+\lambda_j}}e_i\right\rangle^2\label{variance-rewrite-4}\\
    &= \sum_{j=1}^{n_{N}} \rho_j^{-(1-\gamma')} \left\langle U_j\rho_j^{\frac{1}{2}}f_j, G_j\left(\sum_{i=1}^{+\infty} \frac{\mu_i^{2-\beta'}}{\mu_i+\lambda_j} e_i\otimes e_i\right)G_j U_j\rho_j^{\frac{1}{2}}f_j\right\rangle\nonumber\\
    &\lesssim \sum_{j=1}^{n_{N}} j^{\frac{1-\gamma'}{q}} \normx{G_j}^2 \lambda_j^{-\beta'}\normx{U_j\rho_j^{\frac{1}{2}}f_j}^2\label{variance-rewrite-5}
    \end{align}
\end{subequations}
In \Cref{variance-rewrite}, \Cref{variance-rewrite-1} uses the definition of the Hilbert-Schmidt norm; \Cref{variance-rewrite-2} follows from the definition of $\hat{\A}$ (cf.\Cref{our-estimator}) and the fact that for any $j \geq y_N$, we have $\left\langle \rho_j^{\frac{1}{2}}f_j, \left(\hat{\A}-\A_{\lambda}\right) \mu_i^{\frac{1}{2}}e_i\right\rangle = 0$; \Cref{variance-rewrite-4} is obtained from re-arranging and \Cref{variance-rewrite-5} follows from $\normx{\sum_{i=1}^{+\infty}\frac{\mu_i^{2-\beta'}}{\mu_i+\lambda_j} e_i\otimes e_i} = \max_{i\geq 1} \frac{\mu_i^{1-\beta'}}{\mu_i+\lambda_j}\lesssim\lambda_j^{-\beta'}$ and $\rho_j\lesssim j^{-\frac{1}{q}}$.

Note that
\begin{equation}
    \notag
    \begin{aligned}
       U_j &= \left(\C_{KK}+\lambda_j I\right)^{-\half} \left[\hat{\C}_{KL}-\C_{KL} - \left(\hat{\C}_{KK}-\C_{KK}\right)\left( \mathcal{C}_{KK}+\lambda_j I\right)^{-1} \C_{KL}\right] \\
       &=  \frac{1}{N}\sum_{k=1}^N (\mathcal{C}_{KK}+\lambda_j I)^{-\frac{1}{2}} \left[ u_k\otimes v_k - \mathbb{E}_{P_{KL}} u_k\otimes \mathcal{A}_0u_k - (u_k \otimes u_k-\mathbb{E}_{P_{KL}}u_k \otimes u_k) \left(\C_{KK}+\lambda_j I\right)^{-1}\C_{KK} \A_0^* \right] \\
       &= \underbrace{\frac{1}{N}\sum_{k=1}^N  (\mathcal{C}_{KK}+\lambda_j I)^{-\frac{1}{2}}  \left(u_k\otimes (v_k-\mathcal{A}_0u_k)\right)}_{:=U_j^1}\\
       &+ \underbrace{\frac{1}{N}\sum_{k=1}^N  (\mathcal{C}_{KK}+\lambda_j I)^{-\frac{1}{2}}\left[u_k\otimes \mathcal{A}_0u_k-\mathbb{E}_{P_{KL}} u_k\otimes \mathcal{A}_0u_k - (u_k \otimes u_k-\mathbb{E}_{P_{KL}}u_k \otimes u_k) \left(\C_{KK}+\lambda_j I\right)^{-1}\C_{KK} \A_0^* \right]}_{:=U_j^2 = \lambda_j\frac{1}{N}\sum_{k=1}^N  \left(\mathcal{C}_{KK}+\lambda_j I\right)^{-\frac{1}{2}}\left(u_k \otimes \A_0 \left(\C_{KK}+\lambda_j I\right)^{-1} u_k-\mathbb{E}_{P_{KL}}u_k \otimes \A_0 \left(\C_{KK}+\lambda_j I\right)^{-1} u_k\right)}.
    \end{aligned}
\end{equation}
The $U_j^1$ term is the variance of observational noise and $U_j^2$ term is the variance of regularized bias. Thus the $U_j^1$ term is the dominating term. Plugging the above decomposition into \Cref{variance-rewrite}, we deduce that $\mathcal{V} \leq 2\left( \mathcal{V}_1+\mathcal{V}_2\right)$ where
\begin{equation}
    \label{var-series}
    \begin{aligned}
       \mathcal{V}_1 &\lesssim \max_{1\leq j\leq n_N}\normx{G_j}^2 \sum_{j=1}^{n_{N}}j^{\frac{1-\gamma'}{q}} \lambda_j^{-\beta'}\underbrace{ \normx{\frac{1}{N}\sum_{k=1}^N \left[ \left\langle v_k - \A_0 u_k ,\rho_j^{\frac{1}{2}}f_j\right\rangle \left(\mathcal{C}_{KK}+\lambda_j I\right)^{-\frac{1}{2}}u_k \right]}^2}_{:=\mathcal{V}_{1,j}^2} \\
       \mathcal{V}_2 &\lesssim \max_{1\leq j\leq n_N}\normx{G_j}^2 \sum_{j=1}^{n_{N}}j^{\frac{1-\gamma'}{q}} \lambda_j^{2-\beta'}\underbrace{ \normx{\left(\hat{\E}-\E\right)  \left[ \left\langle \A_0 \left( \C_{KK}+\lambda_j I \right)^{-1} u_k ,\rho_j^{\frac{1}{2}}f_j\right\rangle \left(\mathcal{C}_{KK}+\lambda_j I\right)^{-\frac{1}{2}}u_k \right]}^2}_{:=\mathcal{V}_{2,j}^2}
    \end{aligned}
\end{equation}
where $\hat{\mathbb{E}}[X] = \frac{1}{N}\sum_{k=1}^N X_k$ denotes the empirical mean.
Define the event
\begin{equation}
    \notag
    E_{1,j} = \left\{ G_j =  \normx{\left[\mathcal{P}_{i_j}\left( \mathcal{C}_{KK}\right)\right]^{\frac{1}{2}}\left[\mathcal{P}_{i_j}\left( \hat{\mathcal{C}}_{KK}\right)\right]^{\dagger}\left[\mathcal{P}_{i_j}\left( \mathcal{C}_{KK}\right)\right]^{\frac{1}{2}}}  \leq 2\sqrt{a_1}. \right\}.
\end{equation}
Recall that $m_{N} \leq c_0 \left(\frac{N}{\log N}\right)^{\frac{p}{\alpha}}$, by \Cref{conc-covariance}, we know that $E_{1,j}$ holds with probability $\geq 1-2e^{-a_1}$. As a result $E_1 = \cap_{j=1}^{n_N} E_{1,j}$ holds with probability $\geq 1-2n_N e^{-a_1}$. We assume event $E_{1}$ holds in all the following proof.

\textbf{Bounding $\mathcal{V}_1$.} Let
\begin{equation}
    \notag
    X_{j,k} = j^{\frac{1-\gamma'}{2q}}\lambda_j^{-\frac{\beta'}{2}}\left\langle v_{k}-\mathcal{A}_{0} u_{k}, \rho_{j}^{\frac{1}{2}} f_{j}\right\rangle \left(\C_{K K}+\lambda_j I\right)^{-\frac{1}{2}} u_k \in \H_K
\end{equation}
and
$
    X_k = \left( X_{j,k}:1\leq j\leq n_N\right) \in \H_K^{y_N}.
$ Then we have $\mathcal{V}_1 \lesssim \normx{\frac{1}{N}\sum_{k=1}^N X_k}^2$ where the norm here defined for $\H_K^{\otimes y_N}$ is induced by $\left\langle a,b \right\rangle = \sum_{i=1}^{n_N}\left\langle a_i, b_i\right\rangle_{\H_K}$. Note that $X_k, k=1,2,\cdots,N$ are i.i.d. random variables with mean zero, and
\begin{equation}
    \notag
    \begin{aligned}
       &\quad \E\normx{X_1}^{2t}
       = \E_{P_{KL}}\left[ \left(\sum_{j=1}^{n_N} \normx{X_{j,k}}^2\right)^t  \right] \\
       &= \E_{P_{KL}}\left[ \left(\sum_{j=1}^{n_N} j^{\frac{1-\gamma'}{q}}\lambda_j^{-\beta'} \left\langle v_{1}-\mathcal{A}_{0} u_{1}, \rho_{j}^{\frac{1}{2}} f_{j}\right\rangle^2 \normx{\left(\C_{KK}+\lambda_j I\right)^{-\frac{1}{2}}u}^2\right)^t  \right] \\
       &\leq \max_{1\leq j\leq y_N}\sup_{u\in\mathrm{supp}(P_K)}\left( \underbrace{j^{\frac{1-\gamma'}{q}}i_j^{\frac{\beta'}{p}} \normx{\left(\C_{KK}+\lambda_j I\right)^{-\frac{1}{2}}u}^2}_{=:G_1}\right)^{t-1}\cdot\\
       &\quad \underbrace{\E_{(u,v)\sim P_{KL}}\left[  \normx{v-\mathcal{A}_{0} u}^{2t-2}\left(\sum_{j=1}^{n_N} j^{\frac{1-\gamma'}{q}}\lambda_j^{-\beta'} \left\langle v-\mathcal{A}_{0} u, \rho_{j}^{\frac{1}{2}} f_{j}\right\rangle^2 \normx{\left(\C_{KK}+\lambda_j I\right)^{-\frac{1}{2}} u}^2\right) \right]}_{=:G_2}
    \end{aligned}
\end{equation}

By \Cref{uniform-upper} we have 
\begin{equation}
    \notag
    G_1 \lesssim j^{\frac{1-\gamma'}{q}}\lambda_j^{-(\beta'+\alpha)}.
\end{equation}
For $G_2$, note that for fixed $u$, \Cref{asmp-noise} implies that
\begin{equation}
    \notag
    \begin{aligned}
       &\quad \E_{v\mid u}\left[ \normx{v-\mathcal{A}_{0} u}^{2t-2}\left(\sum_{j=1}^{n_N} j^{\frac{1-\gamma'}{q}}i_j^{\frac{\beta'}{p}} \left\langle v-\mathcal{A}_{0} u, \rho_{j}^{\frac{1}{2}} f_{j}\right\rangle^2 \normx{\left(\C_{KK}+\lambda_j I\right)^{-\frac{1}{2}}u}^2\right) \right] \\
       &\leq \frac{1}{2}(2t)! R^{2t-2} \sum_{j=1}^{n_N} \sigma_j^2 j^{\frac{1-\gamma'}{q}}\lambda_j^{-\beta'}\normx{\left(\C_{KK}+\lambda_j I\right)^{-\frac{1}{2}}u}^2.
    \end{aligned}
\end{equation}
where $\sigma_j^2 = \left\langle \rho_j^{\frac{1}{2}}f_j, V\rho^{\frac{1}{2}}f_j\right\rangle$. As a result, we have
\begin{equation}
    \notag
    G_2 \leq \E_{P_K}\left[ \frac{1}{2}(2t)! R^{2t-2} \sum_{j=1}^{n_N} \sigma_j^2 j^{\frac{1-\gamma'}{q}}i_j^{\frac{\beta'}{p}} \normx{\left(\C_{KK}+\lambda_j I\right)^{-\frac{1}{2}}u}^2 \right] \leq \frac{1}{2}(2t)! R^{2t-2} \sigma^2 \max_{1\leq j\leq n_N} j^{\frac{1-\gamma'}{q}}\lambda_j^{-(p+\beta')},
\end{equation}
where in the second step we use $\sum_{j=1}^{+\infty}\sigma_j^2 = \tr{V}=\sigma^2$ and
\begin{equation}
    \notag
    \begin{aligned}
       \E_{P_K}\left[ \normx{\left(\C_{KK}+\lambda_j I\right)^{-\frac{1}{2}}u}^2\right]  &\leq \tr{\E_{P_K}\left[ \left(\C_{KK}+\lambda_j I\right)^{-\frac{1}{2}}u \otimes \left(\C_{KK}+\lambda_j I\right)^{-\frac{1}{2}}u \right]} \\
       &= \tr{\sum_{i=1}^{+\infty}\frac{\mu_i^2}{\mu_i+\lambda_j} e_i\otimes e_i} \\
       &= \sum_{i=1}^{+\infty} \frac{\mu_i}{\mu_i+\lambda_j}\\
       &\lesssim \lambda_j^{-p}.
    \end{aligned}
\end{equation}

We have shown that for some constant $c_1 > 0$, 
\begin{equation}
    \notag
    \E\normx{X_1}^{2t} \leq \frac{1}{2}(2t)!\sigma^2\max_{1\leq j\leq n_N} j^{\frac{1-\gamma'}{q}}\lambda_j^{-(p+\beta')}\cdot \left( c_1 R^2 \max_{1\leq j\leq n_N} j^{\frac{1-\gamma'}{q}}\lambda_j^{-(\beta'+p)} \right)^{t-1}.
\end{equation}
By Bernstein's inequality, the event
\begin{equation}
    \label{bernstein-var1}
    E_{2} := \left\{  \normx{\frac{1}{N} \sum_{k=1}^N X_k}^2 \leq 6 a_2 \left( \frac{\sigma^2\max_{j\in[y_N]} j^{\frac{1-\gamma'}{q}}\lambda_j^{-(\beta'+p)}}{N} + \frac{c_1 R^2 \max_{1\leq j\leq n_N} j^{\frac{1-\gamma'}{q}}\lambda_j^{-(\beta'+\alpha)}}{N^2} \right)  \right\}
\end{equation}

holds with probability $\geq 1-2e^{-a_2}$. By our definition of $\lambda_j$, we have
\begin{equation}
    \notag
    \max_{1\leq j\leq n_N} j^{\frac{1-\gamma'}{q}}\lambda_j^{-(\beta'+p)} \lesssim N^{\max\left\{ \frac{\beta'+p}{\beta+p},\frac{1-\gamma'}{1-\gamma} \right\}}
\end{equation}
and $\lambda_j \gtrsim N^{-\frac{1}{\alpha}}$ (which implies that the $\frac{1}{N^2}$ term is dominated by the $\frac{1}{N}$ term). Hence, under $E_1\cap E_2$ we have
\begin{equation}
    \notag
    \mathcal{V}_1 \lesssim a_1 a_2 \sigma^2 N^{-\min\left\{ \frac{\beta-\beta'}{\beta+p},\frac{\gamma'-\gamma}{1-\gamma} \right\}}
\end{equation}
with probability $\geq 1-2 n_N e^{-a_2}$.

\textbf{Bounding $\mathcal{V}_2$.} For any $j \in\mathbb{Z}_{+}$ we have
\begin{subequations}
    \label{var-4}
    \begin{align}
       &\quad \E_{u\sim P_K}\left[ \left\langle \A_0\left(\C_{KK}+\lambda_j I\right)^{-1}u, \rho_j^{\frac{1}{2}}f_j\right\rangle^2  \right]\nonumber \\
       &= \E_{u\sim P_K} \left\langle \rho_j^{\frac{1}{2}}f_j, \E_{P_K}\left[ \A_0\left(\C_{KK}+\lambda_j I\right)^{-1}u\otimes \A_0\left(\C_{KK}+\lambda_j I\right)^{-1}u\right] \rho_j^{\frac{1}{2}}f_j\right\rangle \nonumber \\
       &= \left\langle \rho_j^{\frac{1}{2}}f_j,  \A_0\left(\C_{KK}+\lambda_j I\right)^{-1} \mathcal{C}_{KK} \left(\C_{KK}+\lambda_j I\right)^{-1}\A_0^* \rho_j^{\frac{1}{2}}f_j\right\rangle \label{var-4-1} \\
       &= \rho_j^{1-\gamma} \left\langle \left(\mathcal{C}_{Q_{L}}^{-\frac{1-\gamma}{2}}\A_0 \mathcal{C}_{KK}^{\frac{1-\beta}{2}}\right)^*\rho_j^{\frac{1}{2}}f_j, 
       \left(\C_{KK}+\lambda_j I\right)^{-1} \mathcal{C}_{KK}^{\beta} \left(\C_{KK}+\lambda_j I\right)^{-1}\left(\mathcal{C}_{Q_{K}}^{-\frac{1-\gamma}{2}}\A_0 \mathcal{C}_{KK}^{\frac{1-\beta}{2}}\right)^*\rho_j^{\frac{1}{2}}f_j\right\rangle \label{var-4-2} \\
       &\lesssim j^{-\frac{1-\gamma}{q}} \lambda_j^{-(2-\beta)} \underbrace{\normx{\left(\mathcal{C}_{Q_{L}}^{-\frac{1-\gamma}{2}}\A_0 \mathcal{C}_{KK}^{\frac{1-\beta}{2}}\right)^*\rho_j^{\frac{1}{2}}f_j}^2}_{=: D_{j,2}} \label{var-4-3}
    \end{align}
\end{subequations}
where \Cref{var-4-1} follows from $\E_{P_K}u\otimes u = \mathcal{C}_{KK}$, \Cref{var-4-2} uses the fact that $\mathcal{C}_{KK}$ and $\C_{KK}+\lambda_j I$ commute, and lastly \Cref{var-4-3} follows from $\normx{\left(\C_{KK}+\lambda_j I\right)^{-1} \mathcal{C}_{KK}^{\beta} \left(\C_{KK}+\lambda_j I\right)^{-1}}_{\H_K} \propto \lambda_j^{-(2-\beta)}$.

Let
\begin{equation}
    \notag
    Y_{j,k} = \left\langle\A_{0}\left(\C_{KK}+\lambda_j I\right)^{-1} u_{k}, \rho_{j}^{\frac{1}{2}} f_{j}\right\rangle \left(\C_{KK}+\lambda_j I\right)^{-\frac{1}{2}} u_k \in \H_K
\end{equation}
and
\begin{equation}
    \notag
    Y_k = \left( Y_{j,k}:1\leq j\leq y_N\right) \in \H_K^{n_N}.
\end{equation}
Then we have
\begin{equation}
    \notag
    \mathcal{V}_2 \lesssim \normx{\frac{1}{N}\sum_{k=1}^N Y_k}_{\H_K^{y_N}}^2.
\end{equation}
Note that $Y_k, k=1,2,\cdots, N$ are i.i.d. random variables, and
\begin{equation}
    \label{var-moment-2-1}
    \begin{aligned}
       &\quad \E\normx{Y_1}^{2t} = \E\left[ \left(\sum_{j=1}^{n_N}\normx{Y_{j,k}}^2\right)^t\right] \\
       &= \E_{P_{K}}\left[\left(\sum_{j=1}^{n_{N}} j^{\frac{1-\gamma^{\prime}}{q}} \lambda_j^{2-\beta'}\left\langle \A_{0}\left(\C_{KK}+\lambda_j I\right)^{-1} u_{1}, \rho_{j}^{\frac{1}{2}} f_{j}\right\rangle^{2}\normx{\mathcal{C}_{K K}^{-\frac{1}{2}} \mathcal{I}_{i_{j}}\left(u_{1}\right)}^{2}\right)^{t}\right] \\
       &\leq \sup_{u\in\mathrm{supp}(P_K)} \left(\sum_{j=1}^{n_{N}} j^{\frac{1-\gamma^{\prime}}{q}} \lambda_j^{2-\beta'}\left\langle \A_{0}\left(\C_{KK}+\lambda_j I\right)^{-1} u, \rho_{j}^{\frac{1}{2}} f_{j}\right\rangle^{2}\normx{\left(\C_{KK}+\lambda_j I\right)^{-\half} u}^{2}\right)^{t-1}\cdot \\
       &\quad \sum_{j=1}^{n_N} j^{\frac{1-\gamma^{\prime}}{q}} \lambda_{j}^{2-\beta'} \E\left[\left\langle \A_{0}\left(\C_{KK}+\lambda_j I\right)^{-1} u, \rho_{j}^{\frac{1}{2}} f_{j}\right\rangle^{2}\right] \sup_{u\in\mathrm{supp}(P_K)} \normx{\left(\C_{KK}+\lambda_j I\right)^{-\half} u}^{2}  \\
       &\lesssim \sup_{u\in\mathrm{supp}(P_K)} \left(\sum_{j=1}^{n_{N}} j^{\frac{1-\gamma^{\prime}}{q}} \lambda_{j}^{2-\beta'-\alpha}\left\langle \A_{0}\left(\C_{KK}+\lambda_j I\right)^{-1} u, \rho_{j}^{\frac{1}{2}} f_{j}\right\rangle^{2}\right)^{t-1}\cdot  \sum_{j=1}^{n_N} j^{\frac{1-\gamma^{\prime}}{q}} \lambda_{j}^{-(\beta'+\alpha-\beta)} D_{j,2}.
    \end{aligned}
\end{equation}

For any $j \in\mathbb{Z}_{+}$ and $u \in\mathrm{supp}(P_K)$ we have
\begin{subequations}
    \label{var-3}
    \begin{align}
       &\quad \sum_{j=1}^{n_{N}} j^{\frac{1-\gamma^{\prime}}{q}} \lambda_{j}^{-(\beta'+\alpha)} \left\langle \A_0\left(\C_{KK}+\lambda_j I\right)^{-1}\lambda_j u, \rho_j^{\frac{1}{2}} f_j\right\rangle^2\nonumber \\
       &\leq \sum_{j=1}^{n_{N}} j^{\frac{1-\gamma^{\prime}}{q}} \lambda_{j}^{-(\beta'+\alpha)}\left(2\left\langle \A_0 u, \rho_j^{\frac{1}{2}}f_j\right\rangle^2 + 2\rho_j^{1-\gamma}\left\langle \mathcal{C}_{Q_{K}}^{-\frac{1-\gamma}{2}}\A_0\left(\C_{KK}+\lambda_j I\right)^{-1}\C_{KK}u, \rho_j^{\frac{1}{2}} f_j\right\rangle^2\right) 
       \label{var-3-1}\\
       &\lesssim \max_{1\leq j\leq y_N} j^{\frac{1-\gamma^{\prime}}{q}} \lambda_{j}^{-(\beta'+\alpha)} + \sum_{j=1}^{n_N} \lambda_j^{-(\beta'+\alpha)}\lambda_j^{-\max\{\alpha-\beta,0\}}  \normx{\left(\mathcal{C}_{Q_{K}}^{-\frac{1-\gamma}{2}}\A_0 \mathcal{C}_{KK}^{\frac{1-\beta}{2}}\right)^* \rho_j^{\frac{1}{2}} f_j}^2\label{var-3-2} \\
       &\lesssim \max_{1\leq j\leq y_N} j^{\frac{1-\gamma^{\prime}}{q}} \lambda_{j}^{-(\beta'+\alpha)} + \max_{1\leq j\leq y_N} \lambda_j^{-(\beta'+\alpha)-\max\{\alpha-\beta,0\}}.\label{var-3-3}
    \end{align}
\end{subequations}
where \Cref{var-3-1} uses the AM-GM inequality, \Cref{var-3-2} follows from the assumption that $||\A_0 u|| \leq A_2$ is uniformly bounded, and that 
\begin{equation}
    \notag
    \normx{\C_{KK}^{-\frac{1-\beta}{2}}\left(\C_{KK}+\lambda_j I\right)^{-1}\C_{KK}u} = \normx{\C_{KK}^{1-\frac{\alpha-\beta}{2}}\left(\C_{KK}+\lambda_j I\right)^{-1}\left(\C_{KK}^{-\frac{1-\alpha}{2}} u\right)} \lesssim \lambda_j^{-\max\{\alpha-\beta,0\}}.
\end{equation}
by \Cref{asmp-bounded-input-data}, and lastly \Cref{var-3-3} follows from $\normx{\A_0}_{\beta,\gamma} \leq B$.
\\

Plugging into \Cref{var-moment-2-1}, we deduce that
\begin{equation}
    \notag
    \begin{aligned}
       &\quad \E \normx{Y_1}^{2t} \\
       &\lesssim \sup_{u\in\mathrm{supp}(P_K)}\left( \max_{1\leq j\leq n_N} j^{\frac{1-\gamma^{\prime}}{q}} \lambda_{j}^{-(\beta'+\alpha)} + \max_{1\leq j\leq n_N} \lambda_j^{-(\beta'+\alpha)-\max\{\alpha-\beta,0\}}\right)^{t-1}\cdot \sum_{j=1}^{n_N} j^{\frac{1-\gamma^{\prime}}{q}} \lambda_{j}^{-(\beta'+\alpha-\beta)} D_{j,2} \\
       &\lesssim \sup_{u\in\mathrm{supp}(P_K)}\left( \max_{1\leq j\leq n_N} j^{\frac{1-\gamma^{\prime}}{q}} \lambda_{j}^{-(\beta'+\alpha)} + \max_{1\leq j\leq n_N} \lambda_j^{-(\beta'+\alpha)-\max\{\alpha-\beta,0\}}\right)^{t-1} \max_{1\leq j \leq n_N} j^{\frac{1-\gamma^{\prime}}{q}} \lambda_{j}^{-(\beta'+\alpha-\beta)}
    \end{aligned}
\end{equation}
where the last step follows from $\sum_{j=1}^{+\infty}D_{j,2} = \normx{\A_0}_{\beta,\gamma}^2$.

By Bernstein's inequality, there exists a constant $C_3$ such that the event
\begin{equation}
    \label{bernstein-var2}
    E_{3} = \left\{ \mathcal{V}_{2} \leq 6 a_3 C_3 \left(\frac{ j^{\frac{1-\gamma'}{q}}\lambda_{j}^{-(\beta'+\alpha-\beta)}}{N}+ \frac{\max_{1\leq j\leq n_N} \lambda_{j}^{-(\beta'+\alpha)} \left(j^{\frac{1-\gamma^{\prime}}{q}} + \lambda_j^{-\max\{\alpha-\beta.0\}} \right) }{N^2} \right)\right\}
\end{equation}
holds with probability $\geq 1-2e^{-a_3}$. 

The definition of $\lambda_N$ implies that the $\frac{1}{N^2}$ term is dominated by the $\frac{1}{N}$ term, so 
\begin{equation}
    \notag
    \mathcal{V}_2 \lesssim a_1 a_3 \frac{1}{N} \max_{1\leq j\leq y_N} j^{\frac{1-\gamma}{q}}\lambda_j^{-(\beta'+\alpha-\beta)} \lesssim N^{-\min\left\{\frac{\beta-\beta'}{\beta+p},\frac{\gamma'-\gamma}{1-\gamma}\right\}}
\end{equation}
holds under $E_1\cap E_3$. To summarize, under $E_1\cap E_2\cap E_3$ which holds with probability $\geq 1-2n_N e^{-a_1}-2e^{-a_2}-2e^{-a_3}$, we have

\begin{equation}
    \notag
    \mathcal{V} \leq 2 a_1\max\{a_2,a_3\}\left(\mathcal{V}_1+\mathcal{V}_2\right) \lesssim N^{-\min\left\{\frac{\beta-\beta'}{\beta+p},\frac{\gamma'-\gamma}{1-\gamma}\right\}}.
\end{equation}
Recall that the bias term is upper bounded in \Cref{bias-upper-bound}. This gives the final upper bound
\begin{equation}
    \notag
    ||\hat{\A}-\A_0||_{\beta',\gamma'} \lesssim N^{-\min\left\{\frac{\beta-\beta'}{2(\beta+p)},\frac{\gamma'-\gamma}{2(1-\gamma)}\right\}}.
\end{equation}

\subsubsection{The hard-learning regime}
\label{sec-hard}
In the previous sections, we focus on the case where $\alpha \leq \beta+p$ and establish an upper bound for the convergence rate via an optimal bias-variance trade-off. The opposite case, $\alpha >\beta+p$ is referred to as the hard-learning regime, for which the optimal rate is not known for several decades even in the case of $\gamma=1$ (cf. the discussion following ~\cite[Theorem 2]{fischer2020sobolev}). In the hard learning regime the $\mathcal{V}_2$ term becomes the leading terms.

In this section, we use the technique developed in previous sections to obtain an upper bound in the hard-learning regime. To do this, we need to re-define the truncation set $S_N$ as follows:
\begin{equation}
    \notag
    S_N = \left\{ (x,y)\in\mathbb{Z}^2 \left|  x^{\frac{\beta'+\alpha-\beta}{p}}y^{\frac{1-\gamma'}{q}} \leq N^{1-\min\left\{ \frac{\beta-\beta'}{\alpha},\frac{\gamma'-\gamma}{1-\gamma}\right\}} \text{ and } x \leq c_0 \left(\frac{N}{\log N}\right)^{\frac{p}{\alpha}}  \right.\right\}.
\end{equation}
The definition implies that the variance can be controlled by $N^{-\min\left\{ \frac{\beta-\beta'}{2\alpha},\frac{\gamma'-\gamma}{2(1-\gamma)}\right\}}$ and it remains to focus on the bias term.

Similar to the derivations in \Cref{sec-bias}, we have
\begin{equation}
    \notag
    || \A_0 - \T_N\left(\A_0\right)||_{\beta',\gamma'}^2 \lesssim \max_{(i,j)\notin S_N} i^{-\frac{\beta-\beta'}{p}} j^{-\frac{\gamma'-\gamma}{q}}.
\end{equation}
The maximum value of the right hand side can be achieved in either of the following two cases:
\begin{itemize}
    \item $i = \O(1)$. Then we have $j \gtrsim N^{\frac{q}{1-\gamma'}\left(1-\min\left\{ \frac{\beta-\beta'}{\alpha},\frac{\gamma'-\gamma}{1-\gamma}\right\}\right)}$ so that
    \begin{equation}
        \notag
        i^{-\frac{\beta-\beta'}{p}} j^{-\frac{\gamma'-\gamma}{q}} \lesssim N^{-\frac{\gamma'-\gamma}{1-\gamma'}\left(1-\min\left\{ \frac{\beta-\beta'}{\alpha},\frac{\gamma'-\gamma}{1-\gamma}\right\}\right)} \leq N^{-\frac{\gamma'-\gamma}{1-\gamma}}.
    \end{equation}
    \item $j = \O(1)$. In this case we must have $i\lesssim N^{\min\left\{ \frac{p}{\alpha},\frac{p}{\beta-\beta'}\frac{\gamma'-\gamma}{1-\gamma} \right\}}$, otherwise it falls into $S_N$ by definition. Hence we have
    \begin{equation}
        \notag
        i^{-\frac{\beta-\beta'}{p}} j^{-\frac{\gamma'-\gamma}{q}} \leq i^{-\frac{\beta-\beta'}{p}} \lesssim N^{-\min\left\{ \frac{\beta-\beta'}{\alpha},\frac{\gamma'-\gamma}{1-\gamma}\right\}}.
    \end{equation}
\end{itemize}
On the other hand, for the variance term we still have $\mathcal{V}_1 \lesssim \frac{1}{N}\max_{1\leq j\leq n_N}j^{\frac{1-\gamma'}{q}}i_j^{\frac{\beta'+p}{p}}$ and $\mathcal{V}_2 \leq \frac{1}{N}\max_{1\leq j\leq n_N} j^{\frac{1-\gamma'}{q}}i_j^{\frac{\beta'+\alpha-\beta}{p}}$, so that
\begin{equation}
    \notag
    \mathcal{V} \lesssim \frac{1}{N}\max_{1\leq j\leq n_N} j^{\frac{1-\gamma'}{q}}i_j^{\frac{\beta'+\alpha-\beta}{p}} \leq N^{-\min\left\{ \frac{\beta-\beta'}{\beta+p},\frac{\gamma'-\gamma}{1-\gamma}\right\}}.
\end{equation}
As a result, we can obtain the following convergence rate:
\begin{equation}
    \notag
    \normx{\hat{\A}-\A_0}_{\beta',\gamma'} \lesssim N^{-\min\left\{ \frac{\beta-\beta'}{2\alpha},\frac{\gamma'-\gamma}{2(1-\gamma)}\right\}}.
\end{equation}

\subsection{Regularization via bias contour}
\label{bias-contour-rate}

In this subsection, we analyze the convergence rate of regularization via bias contour (cf. \Cref{contours_fig}). Specifically, we consider the estimator \Cref{our-estimator} with the choice
\begin{equation}
    \label{bias-contour-coeff-app}
    \lambda_j = \max\left\{ \left(j^{-\frac{\gamma'-\gamma}{q}}N^{\min\left\{ \frac{\beta-\beta'}{\max\{\alpha,\beta+p\}},\frac{\gamma'-\gamma}{1-\gamma} \right\}}\right)^{-\frac{1}{\beta-\beta'}}, c_0\left(\frac{N}{\log N}\right)^{-\frac{1}{\alpha}}  \right\}.
\end{equation}

It now remains to plug the above $\lambda_j$ into our bounds for bias and variance derived in the previous subsections.

\textbf{Bounding the bias term.} It follows from \Cref{bias-ineq} that
\begin{equation}
    \notag
    \begin{aligned}
        \normx{\A_0-\A_{\lambda}}_{\beta,\gamma'}^2
        &\lesssim \max_{1\leq j\leq y_N} j^{-\frac{\gamma'-\gamma}{q}}\lambda_j^{\beta-\beta'} \\
        &\lesssim \max\left\{ N^{-\min\left\{ \frac{\beta-\beta'}{\max\{\alpha,\beta+p\}},\frac{\gamma'-\gamma}{1-\gamma} \right\}}, c_0 \left( \frac{N}{\log N}\right)^{-\frac{\beta-\beta'}{\alpha}} \right\} \\
        &\lesssim N^{-\min\left\{ \frac{\beta-\beta'}{\max\{\alpha,\beta+p\}},\frac{\gamma'-\gamma}{1-\gamma} \right\}}.
    \end{aligned}
\end{equation}

\textbf{Bounding the variance term} It follows from \Cref{bernstein-var1,bernstein-var2} that the variance is bounded by
\begin{equation}
    \notag
    \normx{\hat{\A}-\A_{\lambda}}_{\beta',\gamma'}^2 \lesssim \frac{1}{N}\max_{1\leq j\leq y_N} j^{\frac{1-\gamma'}{q}}\lambda_j^{-\left(\beta'+\max\{\alpha-\beta,p\}\right)}.
\end{equation}
As before, we consider the cases $\alpha\leq\beta+p$ and $\alpha>\beta+p$ separately.
\begin{itemize}
    \item If $\alpha\leq\beta+p$, then it follows that
    \begin{equation}
    \notag
        \begin{aligned}
            \normx{\hat{\A}-\A_{\lambda}}_{\beta',\gamma'}^2 &\lesssim \frac{1}{N}\max_{1\leq j\leq y_N} j^{\frac{1-\gamma'}{q}}\lambda_j^{-\left(\beta'+p\right)} \\
            &\lesssim \frac{1}{N}\max_{1\leq j\leq y_N} j^{\frac{1-\gamma'}{q}}\left(j^{-\frac{\gamma'-\gamma}{q}}N^{\min\left\{ \frac{\beta-\beta'}{\beta+p},\frac{\gamma'-\gamma}{1-\gamma} \right\}}\right)^{\frac{\beta'+p}{\beta-\beta'}} \\
            &\lesssim \frac{1}{N}\max_{1\leq j\leq y_N} j^{\frac{\gamma'-\gamma}{q}\left(\frac{1-\gamma'}{\gamma'-\gamma}-\frac{\beta'+p}{\beta-\beta'}\right)} N^{\frac{\beta'+p}{\beta-\beta'}\min\left\{ \frac{\beta-\beta'}{\beta+p},\frac{\gamma'-\gamma}{1-\gamma} \right\}} \\
            &= \frac{1}{N}\max_{j\in\{1,y_N\}} j^{\frac{\gamma'-\gamma}{q}\left(\frac{1-\gamma'}{\gamma'-\gamma}-\frac{\beta'+p}{\beta-\beta'}\right)} N^{\frac{\beta'+p}{\beta-\beta'}\min\left\{ \frac{\beta-\beta'}{\beta+p},\frac{\gamma'-\gamma}{1-\gamma} \right\}} \\
            &= N^{\min\left\{ \frac{\beta-\beta'}{\beta+p},\frac{\gamma'-\gamma}{1-\gamma} \right\}\max\left\{\frac{\beta'+p}{\beta-\beta'},\frac{1-\gamma'}{\gamma'-\gamma}\right\}-1} \\
            &= N^{-\min\left\{ \frac{\beta-\beta'}{\beta+p},\frac{\gamma'-\gamma}{1-\gamma} \right\}},
        \end{aligned}
    \end{equation}
    where we use $y_N^{\frac{\gamma'-\gamma}{q}}=N^{\min\left\{ \frac{\beta-\beta'}{\beta+p},\frac{\gamma'-\gamma}{1-\gamma} \right\}}$ by definition.
    \item If $\alpha>\beta+p$, then similarly we have
    \begin{equation}
        \notag
        \begin{aligned}
            \normx{\hat{\A}-\A_{\lambda}}_{\beta',\gamma'}^2 &\lesssim \frac{1}{N}\max_{1\leq j\leq y_N} j^{\frac{1-\gamma'}{q}}\lambda_j^{-\beta'+\alpha-\beta} \\
            &\lesssim \frac{1}{N}\max_{1\leq j\leq y_N} j^{\frac{1-\gamma'}{q}}\left(j^{-\frac{\gamma'-\gamma}{q}}N^{\min\left\{ \frac{\beta-\beta'}{\alpha},\frac{\gamma'-\gamma}{1-\gamma} \right\}}\right)^{\frac{\beta'+\alpha-\beta}{\beta-\beta'}} \\
            &= \frac{1}{N}\max_{j\in\{1,y_N\}} j^{\frac{1-\gamma'}{q}}\left(j^{-\frac{\gamma'-\gamma}{q}}N^{\min\left\{ \frac{\beta-\beta'}{\alpha},\frac{\gamma'-\gamma}{1-\gamma} \right\}}\right)^{\frac{\beta'+\alpha-\beta}{\beta-\beta'}} \\
            &\leq N^{-\min\left\{ \frac{\beta-\beta'}{\alpha},\frac{\gamma'-\gamma}{1-\gamma} \right\}}.
        \end{aligned}
    \end{equation}
\end{itemize}
Hence we deduce that 
\begin{equation}
    \notag
    \normx{\hat{\A}-\A_{\lambda}}_{\beta',\gamma'}^2 \lesssim N^{-\min\left\{ \frac{\beta-\beta'}{\alpha},\frac{\gamma'-\gamma}{1-\gamma} \right\}},
\end{equation}
as desired.

\subsection{Implication of the upper bound}
\label{subsec:implication}
In this section, we discuss the implications of our upper bounds under the $(\beta',\gamma')$-norm.

Note that
$\normx{\mathcal{C}_{Q_{L}}^{-\frac{1-\gamma'}{2}}v}_{\H_L} = \normx{v}_{\H_L^{2-\gamma'}}$
for all $v \in L_2(Q_{L})$ (if one side of the equation is $+\infty$ then so is the other), we have that
\begin{equation}
    \label{bound-exp}
    \begin{aligned}
       \E_{u\sim P_{K}} \normx{\left(\hat{\A}-\A_0\right)u}_{\H_L^{2-\gamma'}}^2 
       &= \E_{u\sim P_{K}} \normx{\mathcal{C}_{Q_{L}}^{-\frac{1-\gamma'}{2}}\left(\hat{\A}-\A_0\right)u}_{\H_L}^2 \\
       &= \tr{\mathcal{C}_{Q_{L}}^{-\frac{1-\gamma'}{2}}\left(\hat{\A}-\A_0\right) \E_{u\sim P_K} u\otimes u\left(\mathcal{C}_{Q_{L}}^{-\frac{1-\gamma'}{2}}\left(\hat{\A}-\A_0\right)\right)^*}\\
       &\lesssim  \normx{\hat{\A}-\A_0}_{\beta',\gamma'}^2,
    \end{aligned}
\end{equation}
where the last step follows from $\E_{u\sim P_K}u\otimes u = \mathcal{C}_{P_K}$. Note that the above derivations hold for any $0\leq \beta' < \beta$, so choosing $\beta' = 0$ yields the best upper bound. We can see from \Cref{bound-exp} that our analysis implies an upper bound of the expected error of the learned solution evaluated under the $\H_L^{2-\gamma'}$ norm. On the other hand, it is also possible to obtain a \textit{uniform} convergence rate when $\beta' \geq \alpha$:
\begin{equation}
    \notag
    \begin{aligned}
        \normx{\left(\hat{\A}-\A_0\right)u}_{\H_L^{2-\gamma'}} &= \normx{\mathcal{C}_{Q_{L}}^{-\frac{1-\gamma'}{2}}\left(\hat{\A}-\A_0\right)u}_{\H_L} \\&\leq  \normx{\hat{\A}-\A_0}_{\beta',\gamma'}\cdot\normx{\mathcal{C}_{P_K}^{-\frac{1-\beta'}{2}}u}_{\H_K} \lesssim \normx{\hat{\A}-\A_0}_{\beta',\gamma'}.
    \end{aligned}
\end{equation}

\section{Proofs for the multi-level operator learning algorithm}
\label{appsec:multilevel}
In this section, we analyze the convergence rate of our multi-level algorithm described in \Cref{sec-multilevel}. We define $\eta_1 = \min \left\{\frac{\beta-\beta^{\prime}}{\max \{\alpha, \beta+p\}}, \frac{\gamma^{\prime}-\gamma}{1-\gamma}\right\}$ and $\eta_2 = \max\left\{1-\frac{\beta-\beta^{\prime}}{\max \{\alpha, \beta+p\}}, \frac{1-\gamma^{\prime}}{1-\gamma}\right\}=1-\eta_1$. We first restrict ourselves to the case when $\frac{\beta-\beta^{\prime}}{\max \{\alpha, \beta+p\}}\neq \frac{\gamma^{\prime}-\gamma}{1-\gamma}$; the special case when the two terms are equal will be separately treated in \Cref{subsec:multilevel-special}. For the optimal bias and variance contours $\ell_{C_1,\mathtt{bias}}$ and $\ell_{C_2,\mathtt{var}}$ with $C_1 = N^{\eta_1}$ and $C_2 = N^{\eta_2}$, we define a sequence $\{x_n\}$ as follows:
\begin{subequations}
    \label{multilevel-update}
    \begin{align}
        x_0 &= \max\left\{ \half N^{\frac{p}{\beta'+p}\eta_2}, c_0\left(\frac{N}{\log N}\right)^{-\frac{1}{\alpha}} \right\} \label{multilevel-update-x0} \\
        y_n &= \text{the solution of } x_n^{\frac{\beta'+\max\{\alpha-\beta,p\}}{p}} y^{\frac{1-\gamma'}{q}} = N^{\eta_2},\quad n\geq 0 \label{multilevel-update-y}\\
        x_{n+1} &= \text{the solution of } x^{\frac{\beta-\beta'}{p}}y_n^{\frac{\gamma'-\gamma}{q}} = N^{\eta_1},\quad n\geq 0.\label{multilevel-update-x}
    \end{align}
\end{subequations}

We first derive an explicit recursive formula for $\{x_n\}$.

\begin{lemma}
    \label{multilevel-recursion}
    Let $u = \frac{\beta'+\max\{\alpha-\beta,p\}}{\beta-\beta'}\frac{\gamma'-\gamma}{1-\gamma'}>0$, then
    \begin{itemize}
        \item if $u > 1$, then
        \begin{equation}
            \notag
            N^{-\frac{p}{\beta+p}}x_{n+1} = \left(N^{-\frac{p}{\beta+p}}x_{n}\right)^{u}.
        \end{equation}
        \item if $u < 1$, then
        \begin{equation}
            \notag
            x_{n+1} = x_n^u.
        \end{equation}
    \end{itemize}
\end{lemma}

\begin{proof}
    \begin{itemize}
        \item Suppose that $u>1$, then we have $\eta_1 = \frac{\beta-\beta'}{\max\{\alpha,\beta+p\}}$ and $\eta_2 = 1-\eta_1$. It follows from \Cref{multilevel-update-x,multilevel-update-y} that
        \begin{equation}
            \notag
            \begin{aligned}
                x_{n+1} &= N^{\frac{p}{\max\{\alpha,\beta+p\}}} y_n^{-\frac{\gamma'-\gamma}{q}\frac{p}{\beta-\beta'}} \\
                &= N^{\frac{p}{\max\{\alpha,\beta+p\}}} \left( N^{\eta_2}x_n^{-\frac{\beta'+\max\{\alpha-\beta,p\}}{p}}\right)^{-\frac{\gamma'-\gamma}{1-\gamma'}\frac{p}{\beta-\beta'}} \\
                &= N^{\frac{p}{\max\{\alpha,\beta+p\}}} \left(N^{-\frac{p}{\max\{\alpha,\beta+p\}}}x_{n}\right)^{u}.
            \end{aligned}
        \end{equation}
        \item Suppose that $u<1$, then we have $\eta_1 = \frac{\gamma'-\gamma}{1-\gamma}$ and $\eta_2=\frac{1-\gamma'}{1-\gamma}$, so that $\frac{\eta_1}{\eta_2} = \frac{\gamma'-\gamma}{1-\gamma'}$, and it follows from \Cref{multilevel-update-x,multilevel-update-y} that $x_n^{\frac{\beta'+\max\{\alpha-\beta,p\}}{p}\frac{\gamma'-\gamma}{1-\gamma'}} = x_{n+1}^{\frac{\beta-\beta'}{p}}$, thus $x_{n+1}=x_n^u$.
    \end{itemize}
\end{proof}

\Cref{multilevel-recursion} implies that when $u\neq 1$, the sequence $\{x_n\}$ decreases super-exponentially. Thus, there exists $L_N = \O(\log\log N)$ such that $x_n \leq 2$ for all $n\geq L_N$.

Let $\lambda_i^{(K)} = x_i^{-\frac{1}{p}}$ and $\lambda_i^{(L)}=y_i^{-\frac{1}{q}}$, then we construct the following estimator:
\begin{equation}
    \label{multilevel-estimator-app}
    \hat{\A}_{\mathtt{ml}} =  \sum_{i=0}^{L_N}\left(\sum_{y_{i-1}\leq j < y_{i}}\rho_j^{\half}f_j\otimes\rho_j^{\half}f_j\right)\hat{\C}_{YX}\left(\hat{\C}_{KK}+\lambda_i^{(K)}I\right)^{-1}
\end{equation}
where $y_{-1}:=0$.
Note that each summand in the above equation is essentially a regularized least-squares estimator and learns a rectangular region. The following theorem states that the estimator $\hat{\A}_{\mathtt{ml}}$ can achieve minimax optimal convergence rate.

\begin{theorem}
    \label{thm-upper-multilavel-app}
    Consider the estimator $\hat{\A}_{\mathtt{ml}}$ defined by \Cref{multilevel-estimator}. Suppose that \Cref{asmp-eigendecay,asmp-bounded-input-data,asmp-bounded-output-data,asmp-noise,asmp-beta-gamma-norm} hold, then there exists a universal constant $C$, such that
    \begin{equation}
        \notag
        \normx{\hat{\A}_{\mathtt{ml}}-\A_0}_{\beta',\gamma'}^2 \leq C \tau^2 \left(\frac{N}{\log N}\right)^{-\min\left\{\frac{\beta-\beta'}{\max\{\alpha,\beta+p\}},\frac{\gamma'-\gamma}{1-\gamma}\right\}} \log^2 N
    \end{equation}
    holds with probability $\geq 1-e^{-\tau}$.
\end{theorem}

\begin{proof}
    The proof of \Cref{thm-upper-multilavel} is similar to that of \Cref{thm-upper,thm-upper-bias}. We consider the bias-variance decomposition of the estimation error
    \begin{equation}
        \notag
        \normx{\hat{\A}_{\mathtt{ml}}-\A_0}_{\beta',\gamma'} \leq \normx{\hat{\A}_{\mathtt{ml}}-\hat{\A}_{\mathtt{ml}}^{\lambda}}_{\beta',\gamma'}+\normx{\hat{\A}_{\mathtt{ml}}^{\lambda}-\A_0}_{\beta',\gamma'}
    \end{equation}
    where 
    \begin{equation}
        \label{multilevel-regularized}
        \hat{\A}_{\mathtt{ml}}^{\lambda} = \sum_{i=0}^{L_N}\left(\sum_{y_i\leq j < y_{i+1}}\rho_j^{\half}f_j\otimes\rho_j^{\half}f_j\right) \C_{YX}\left(\C_{KK}+\lambda_i^{(K)}I\right)^{-1}.
    \end{equation}
    \textbf{Bounding the bias term.} Since $\normx{\A_0}_{\beta,\gamma}\leq B$, we can write
    \begin{equation}
        \notag
        \A_0 := \sum_{i=1}^{+\infty}\sum_{j=1}^{+\infty}a_{ij}\mu_i^{\frac{\beta}{2}}\rho_j^{1-\frac{\gamma}{2}} f_j \otimes e_i
    \end{equation}
    where the coefficient matrix $A_0 = (a_{ij})_{1\leq i,j\leq +\infty}$ satisfies $\normx{A_0}_{F}^2 \leq B^2$.
    We fix $(i,j)\in\mathbb{Z}_+^2$ and assume WLOG that $y_{m_j-1} \leq j < y_{m_j}$ for some $m\geq 0$, where $y_{L_N+1}=+\infty$. It follows from \Cref{multilevel-regularized} that
    \begin{equation}
        \notag
        \begin{aligned}
             \left\langle \rho_j^{\half}f_j, \hat{\A}_{\mathtt{ml}}^{\lambda} \mu_i^{\half}e_i\right\rangle 
            &= \sum_{k=0}^{L_N} \left\langle \left(\sum_{y_{k-1}\leq j < y_{k}}\rho_j^{\half}f_j\otimes\rho_j^{\half}f_j\right) \rho_j^{\half}f_j, \C_{YX}\left(\C_{KK}+\lambda_k^{(K)}I\right)^{-1} \mu_i^{\half} e_i\right\rangle \\
            &= \frac{\mu_i}{\mu_i+\lambda_m^{(K)}}\rho_j^{\frac{1-\gamma}{2}}\mu_i^{-\frac{1-\beta}{2}} a_{ij}.
        \end{aligned}
    \end{equation}
    Thus
    \begin{equation}
        \label{multilevel-bias-ineq}
        \begin{aligned}
            \normx{\A_0-\hat{\A}_{\mathtt{ml}}^{\lambda}}_{\beta',\gamma'}^2 &= \normx{\C_{Q_L}^{-\frac{1-\gamma}{2}}\left(\hat{\A}_{\mathtt{ml}}^{\lambda}-\A_0\right)\C_{KK}^{\frac{1-\beta'}{2}}}_{\HS}^2 \\
            &= \sum_{i,j=1}^{+\infty}\left\langle \rho_j^{\half}f_j, \C_{Q_L}^{-\frac{1-\gamma'}{2}}\left(\hat{\A}_{\mathtt{ml}}^{\lambda}-\A_0\right)\C_{KK}^{\frac{1-\beta'}{2}} \mu_i^{\half} e_i\right\rangle^2 \\
            &= \sum_{i,j=1}^{+\infty} \left(\frac{\lambda_{m_j}^{(K)}}{\mu_i+\lambda_{m_j}^{(K)}}\right)^2 \mu_i^{\beta-\beta'}\rho_j^{\gamma'-\gamma} a_{ij}^2 \\
            &= \sum_{j=1}^{+\infty} \rho_j^{\gamma'-\gamma}\left(\sum_{i=1}^{+\infty} a_{ij}^2\right)\max_{i\geq 1}\mu_i^{\beta-\beta'}\left(\frac{\lambda_{m_j}^{(K)}}{\mu_i+\lambda_{m_j}^{(K)}}\right)^2 \\
            &\lesssim \sum_{j=1}^{+\infty}\rho_j^{\gamma'-\gamma}\left(\lambda_{m_j}^{(K)}\right)^{\beta-\beta'}\left(\sum_{i=1}^{+\infty} a_{ij}^2\right) \lesssim B^2\max_{j\geq 1}\rho_j^{\gamma'-\gamma}\left(\lambda_{m_j}^{(K)}\right)^{\beta-\beta'} \\
            &\leq B^2 \max_{j\geq 1} \rho_j^{\gamma'-\gamma} x_{m_j}^{-\frac{\beta-\beta'}{p}} 
            \lesssim B^2 \max_{j\geq 1} j^{-\frac{\gamma'-\gamma}{q}} x_{m_j}^{-\frac{\beta-\beta'}{p}} \\
            &\leq B^2 y_{m_j-1}^{-\frac{\gamma'-\gamma}{q}}x_{m_j}^{-\frac{\beta-\beta'}{p}} \lesssim N^{-\eta_1}
        \end{aligned}
    \end{equation}
    where we recall that $\eta_1=\min \left\{\frac{\beta-\beta^{\prime}}{\max \{\alpha, \beta+p\}}, \frac{\gamma^{\prime}-\gamma}{1-\gamma}\right\}$ and the last step follows from \Cref{multilevel-update-x}.

    \textbf{Bounding the variance term.} The variance term can be rewritten in the following way:
    \begin{equation}
        \label{variance-rewrite-multilevel}
        \begin{aligned}
        \mathcal{V} &= \normx{ \hat{\A}_{\mathtt{ml}}-\A_{\mathtt{ml}}^{\lambda}}_{\beta',\gamma'}^2 \nonumber\\
        &= \normx{ C_{Q_L}^{-\frac{1-\gamma'}{2}}\left( \hat{\A}_{\mathtt{ml}}-\A_{\mathtt{ml}}^{\lambda}\right) C_{KK}^{\frac{1-\beta'}{2}}}_{\HS}^2\\
        &= \sum_{i,j=1}^{+\infty} \left\langle \rho_j^{\frac{1}{2}}f_j, C_{Q_L}^{-\frac{1-\gamma'}{2}}\left( \hat{\A}_{\mathtt{ml}}-\A_{\mathtt{ml}}^{\lambda}\right) C_{KK}^{\frac{1-\beta'}{2}} \mu_i^{\frac{1}{2}}e_i\right\rangle^2 \\
        &= \sum_{j=1}^{z_{N}} \rho_j^{-(1-\gamma')}\sum_{i=1}^{+\infty} \left\langle \rho_j^{\frac{1}{2}}f_j,  \left[\hat{\C}_{LK}\left( \hat{\mathcal{C}}_{KK}+\lambda_{m_j} I\right)^{-1}-\C_{LK}\left( \C_{KK}+\lambda_{m_j} I\right)^{-1}\right]\mu_i^{1-\frac{\beta'}{2}}e_i\right\rangle^2  \\
        &= \sum_{j=1}^{z_{N}} \rho_j^{-(1-\gamma')}\sum_{i=1}^{+\infty} \left\langle \underbrace{\left(\C_{KK}+\lambda_{m_j} I\right)^{-\frac{1}{2}}\left[ \hat{\C}_{KL} - \left(\hat{\mathcal{C}}_{KK}+\lambda_{m_j} I\right) \left( \mathcal{C}_{KK}+\lambda_{m_j} I\right)^{-1} \C_{KL} \right]}_{=: U_{m_j}} \rho_j^{\frac{1}{2}}f_j, \right.\\ 
        &\qquad \left.\underbrace{\left( \mathcal{C}_{KK}+\lambda_{m_j} I\right)^{\frac{1}{2}}\left( \hat{\mathcal{C}}_{KK}+\lambda_{m_j} I\right)^{-1}\left(\mathcal{C}_{KK}+\lambda_{m_j} I\right)^{\frac{1}{2}}}_{=: G_{m_j}}\frac{\mu_i^{1-\frac{\beta'}{2}}}{\sqrt{\mu_i+\lambda_j}}e_i\right\rangle^2\\
        &= \sum_{j=1}^{z_{N}} \rho_j^{-(1-\gamma')} \left\langle U_{m_j}\rho_j^{\frac{1}{2}}f_j, G_{m_j}\left(\sum_{i=1}^{+\infty} \frac{\mu_i^{2-\beta'}}{\mu_i+\lambda_{m_j}} e_i\otimes e_i\right)G_{m_j} U_{m_j}\rho_j^{\frac{1}{2}}f_j\right\rangle\nonumber\\
        &\lesssim \sum_{j=1}^{z_{N}} j^{\frac{1-\gamma'}{q}} \normx{G_{m_j}}^2 \lambda_{m_j}^{-\beta'}\normx{ U_{m_j}\rho_j^{\frac{1}{2}}f_j}^2
        \end{aligned}
    \end{equation}
    for reasons similar to \Cref{variance-rewrite}. It now remains to bound $\normx{G_{m_j}}$ and $\normx{U_{m_j}\rho_j^{\half}f_j}$ for $1\leq j\leq L_N$. Note that these quantities have already been bounded in \Cref{sec-var} with $\lambda_{m_j}$ replaced with $\lambda_j$ (there we use a different regularization for each $j$). Hence, those bounds can be directly applied here, so there exists a constant $C>0$ such that
    \begin{equation}
        \notag
        \mathcal{V} \leq C a^2 \frac{1}{N} \max_{1\leq j\leq L_N} j^{\frac{1-\gamma'}{q}}\lambda_{m_j}^{-\left(\beta'+\max\{\alpha-\beta,p\}\right)}
    \end{equation}
    with probability $\geq 1-Ne^{-a}$.
    Since $j \leq y_{m_j}$, by \Cref{multilevel-update-y} we have
    \begin{equation}
        \notag
        j^{\frac{1-\gamma'}{q}}\lambda_{m_j}^{-\left(\beta'+\max\{\alpha-\beta,p\}\right)} \lesssim y_{m_j}^{\frac{1-\gamma'}{q}} x_{m_j}^{\frac{\beta'+\max\{\alpha-\beta,p\}}{p}} = N^{\eta_2}.
    \end{equation}
    Hence
    \begin{equation}
        \notag
        \mathcal{V} \lesssim \frac{1}{N} \max_{1\leq j\leq L_N} j^{\frac{1-\gamma'}{q}}\lambda_{m_j}^{-\left(\beta'+\max\{\alpha-\beta,p\}\right)} \leq N^{\eta_2-1} = N^{\eta_1}.
    \end{equation}
    Combining the bias and variance bounds, the conclusion directly follows.
\end{proof}

\subsection{Special case: $\frac{\beta-\beta^{\prime}}{\max \{\alpha, \beta+p\}}= \frac{\gamma^{\prime}-\gamma}{1-\gamma}$}
\label{subsec:multilevel-special}

Note that \Cref{multilevel-recursion} does not cover the case $u=1$, or equivalently $\frac{\beta-\beta^{\prime}}{\max \{\alpha, \beta+p\}}= \frac{\gamma^{\prime}-\gamma}{1-\gamma}$. This case is special since the bias contour coincides with the variance contour, and we need to modify our construction of the multilevel estimator.

We define two sequences $\{x_n\},\{y_n\}$ as follows:
\begin{equation}
    \label{multilevel-special-recursion}
    \begin{aligned}
        x_0&=\max \left\{\frac{1}{2} N^{\frac{p}{\beta^{\prime}+p} \eta_2}, c_0\left(\frac{N}{\log N}\right)^{-\frac{1}{\alpha}}\right\} \\
        x_{n} &= \half x_{n-1} \\
        y_n &= \text{the solution of } x_n^{\frac{\beta-\beta'}{p}}y_n^{\frac{\gamma'-\gamma}{q}}=N^{\eta_1},
    \end{aligned}
\end{equation}
where we recall that $\eta_1=\frac{\beta-\beta'}{\max\{\alpha,\beta+p\}}=\frac{\gamma'-\gamma}{1-\gamma}$. In this case, there exists $L_N = \O(\ln N)$ such that $x_n < 1$ for all $n\geq L_N$. Let $\lambda_i^{(K)}=x_i^{-\frac{1}{p}}$, then we construct the following estimator:
\begin{equation}
    \label{multilevel-estimator-special-case}
    \begin{aligned}
        \hat{\mathcal{A}}_{\mathtt{ml} }^{\lambda}=\sum_{i=0}^{L_N}\left(\sum_{y_{i-1} \leqslant j<y_i} \rho_j^{\frac{1}{2}} f_j \otimes \rho_j^{\frac{1}{2}} f_j\right) \hat{\mathcal{C}}_{L K}\left(\hat{\mathcal{C}}_{K K}+\lambda_i^{(K)} I\right)^{-1}.
    \end{aligned}
\end{equation}

Similar to \Cref{thm-upper-multilavel-app}, we can establish the following result:

\begin{theorem}
    \label{thm-upper-multilavel-special-app}
    Consider the estimator $\hat{\A}_{\mathtt{ml}}$ defined by \Cref{multilevel-estimator-special-case}. Suppose that \Cref{asmp-eigendecay,asmp-bounded-input-data,asmp-bounded-output-data,asmp-noise,asmp-beta-gamma-norm} hold, then there exists a universal constant $C$, such that
    \begin{equation}
        \notag
        \normx{\hat{\A}_{\mathtt{ml}}-\A_0}_{\beta',\gamma'}^2 \leq C \tau^2 \left(\frac{N}{\log N}\right)^{-\min\left\{\frac{\beta-\beta'}{\max\{\alpha,\beta+p\}},\frac{\gamma'-\gamma}{1-\gamma}\right\}} \log^2 N
    \end{equation}
    holds with probability $\geq 1-e^{-\tau}$.
\end{theorem}

\begin{proof}
    The proof of \Cref{thm-upper-multilavel-special-app} is similar to that of \Cref{thm-upper,thm-upper-bias}. We consider the bias-variance decomposition
    \begin{equation}
        \notag
        \normx{\hat{\mathcal{A}}_{\mathtt{ml}}-\mathcal{A}_0}_{\beta^{\prime}, \gamma^{\prime}} \leqslant\normx{\hat{\mathcal{A}}_{m 1}-\hat{\mathcal{A}}_{m 1}^\lambda}_{\beta^{\prime}, \gamma^{\prime}}+\normx{\hat{\mathcal{A}}_{m 1}^\lambda-\mathcal{A}_0}_{\beta^{\prime}, \gamma^{\prime}}
    \end{equation}
    where
    \begin{equation}
        \label{multilevel-special-regularized}
        \begin{aligned}
            \mathcal{A}_{\mathtt{ml} }^{\lambda}=\sum_{i=0}^{L_N}\left(\sum_{y_{i-1} \leqslant j<y_i} \rho_j^{\frac{1}{2}} f_j \otimes \rho_j^{\frac{1}{2}} f_j\right) \mathcal{C}_{L K}\left(\mathcal{C}_{K K}+\lambda_i^{(K)} I\right)^{-1}.
        \end{aligned}
    \end{equation}
    as defined in \Cref{multilevel-estimator-special-case}.

    \textbf{Bounding the bias term.} Let $\mathcal{A}_0:=\sum_{i=1}^{+\infty} \sum_{j=1}^{+\infty} a_{i j} \mu_i^{\frac{\beta}{2}} \rho_j^{1-\frac{\gamma}{2}} f_j \otimes e_i$ with coefficient matrix $A_0=(a_{ij})_{i,j=1}^{+\infty}$ such that $\normx{A_0}_F^2\leq B^2$. We fix $(i,j)\in\mathbb{Z}_+^2$ and assume WLOG that $y_{m_j-1}\leq j< y_{m_j}$ for some $m_j\geq 0$, where $y_{L_N+1}=+\infty$. It follows from \Cref{multilevel-special-regularized} that
    \begin{equation}
        \notag
        \left\langle \rho_j^{\half}f_j,\A_{\mathtt{ml}}^{\lambda}\mu_i^{\half}e_i\right\rangle = \frac{\mu_i}{\mu_i+\lambda_{m_j}^{(K)}}\rho_j^{\frac{1-\gamma}{2}}\mu_i^{-\frac{1-\beta}{2}}a_{ij}.
    \end{equation}
    Thus we can proceed as in \Cref{multilevel-bias-ineq} to deduce that
    \begin{equation}
        \notag
        \begin{aligned}
            \normx{\A_0-\A_{\mathtt{ml}}^{\lambda}}_{\beta',\gamma'}^2 &\leq 
            \max_{j\geq 1}\rho_j^{\gamma'-\gamma}\left(\lambda_{m_j}^{(K)}\right)^{\beta-\beta'} \\
            &\lesssim \max\left\{\max_{1\leq j\leq L_N} j^{-\frac{\gamma'-\gamma}{q}}x_{m_j}^{-\frac{\beta-\beta'}{p}}, y_{L_N}^{-\frac{\gamma'-\gamma}{q}}\right\} \\
            &\leq \max\left\{\max_{1\leq j\leq L_N} y_{m_j-1}^{-\frac{\gamma'-\gamma}{q}}x_{m_j}^{-\frac{\beta-\beta'}{p}}, y_{L_N}^{-\frac{\gamma'-\gamma}{q}}\right\}
        \end{aligned}
    \end{equation}
    The definition \Cref{multilevel-special-recursion} implies that
    \begin{equation}
        \notag
        y_{m_j-1}^{-\frac{\gamma'-\gamma}{q}}x_{m_j}^{-\frac{\beta-\beta'}{p}} \leq 2^{\frac{\beta-\beta'}{p}} y_{m_j-1}^{-\frac{\gamma'-\gamma}{q}}x_{m_j-1}^{-\frac{\beta-\beta'}{p}} \leq 2^{\frac{\beta-\beta'}{p}} N^{-\eta_1}.
    \end{equation}
    On the other hand, since $x_{L_N}<1$, by \Cref{multilevel-special-recursion} implies that $y_{L_N}^{-\frac{\gamma'-\gamma}{q}} \lesssim N^{-\eta_1}$. Therefore, for the bias term $\normx{\A_0-\A_{\mathtt{ml}}^{\lambda}}_{\beta',\gamma'}^2 \lesssim N^{-\eta_1}$.

    \textbf{Bounding the variance term.} Repeating the arguments in \Cref{variance-rewrite-multilevel}, we can deduce that there exists a constant $C>0$ such that
    \begin{equation}
        \notag
        \mathcal{V} \leq C a^2 \frac{1}{N} \max_{1\leq j\leq L_N} j^{\frac{1-\gamma'}{q}}\lambda_{m_j}^{-\left(\beta'+\max\{\alpha-\beta,p\}\right)} \leq C a^2 \frac{1}{N} \max_{1\leq j\leq L_N} y_{m_j}^{\frac{1-\gamma'}{q}}x_{m_j}^{\frac{\left(\beta'+\max\{\alpha-\beta,p\}\right)}{p}} \lesssim N^{-\eta_1}
    \end{equation}
    with probability $\geq 1-Ne^{-a}$.

    Combining the bias and variance bounds, we arrive at the desired conclusion.
\end{proof}

The conclusion of \Cref{thm-upper-multilavel} then follows from \Cref{thm-upper-multilavel-app,thm-upper-multilavel-special-app}.

\section{Auxiliary results}
\label{appsec:conc-ineq}
\begin{lemma}
    \label{lem-norm-equiv}
    We have $||T||_{\beta,\gamma} = || \mathcal{C}_{Q_L}^{-(1-\gamma)/2}\circ T\circ \mathcal{C}_{KK}^{(1-\beta)/2}||_{\HS\left(\H_K,\H_{L}\right)}$.
\end{lemma}

\begin{proof}
    We recall from the definition that $||T||_{\beta,\gamma}=||\left(I_{1,\gamma,Q_L}\right)^{\dagger}\circ T \circ I_{1\beta,P_K}^*||_{\HS(\H_K^{\beta},\H_L^{\gamma})}$, so that
    \begin{equation}
        \notag
        \begin{aligned}
            ||T||_{\beta,\gamma}^2 
            &= ||\left(I_{1,\gamma,Q_L}\right)^{\dagger}\circ T \circ I_{1\beta,P_K}^*||_{\HS(\H_K^{\beta},\H_L^{\gamma})}^2 \\
            &=\sum_{i,j=1}^{+\infty} \left\langle \rho_j^{\frac{\gamma}{2}}f_j, \left(I_{1,\gamma,Q_L}^*\right)^{\dagger}\circ T \circ I_{1,\beta,P_K}^* \mu_i^{\frac{\beta}{2}}e_i \right\rangle_{\H_L^\gamma}^2\\
            &= \sum_{i,j=1}^{+\infty} \left\langle \rho_j^{\frac{\gamma}{2}}f_j, \left(I_{1,\gamma,Q_L}^*\right)^{\dagger}\circ T \mu_i^{1-\frac{\beta}{2}}e_i \right\rangle_{\H_L^\gamma}^2\\
            &= \sum_{i,j=1}^{+\infty} \left\langle \rho_j^{\frac{\gamma}{2}}f_j, T \mu_i^{1-\frac{\beta}{2}}e_i \right\rangle_{\H_L}^2\\
            &= \sum_{i,j=1}^{+\infty} \left\langle \rho_j^{\half}f_j,\mathcal{C}_{Q_L}^{-(1-\gamma)/2}\circ T\circ \mathcal{C}_{KK}^{(1-\beta)/2} \mu_i^{\half}e_i\right\rangle_{\H_L}^2\\
            &= ||\mathcal{C}_{Q_L}^{-(1-\gamma)/2}\circ T\circ \mathcal{C}_{KK}^{(1-\beta)/2}||_{\HS}^2
        \end{aligned}
    \end{equation}
    as desired.
\end{proof}

\begin{lemma}
\label{uniform-upper}
Under \Cref{asmp-bounded-input-data}, we have
\begin{equation}
    \notag
    || \left(C_{KK}+\lambda I\right)^{-\frac{1}{2}} u || \leq \lambda^{-\frac{\alpha}{2}}\cdot A_1 \quad P_K\text{-a.s.}
\end{equation}
\end{lemma}

\begin{proof}
By \Cref{asmp-bounded-input-data} we have $|| \C_{KK}^{-\frac{1-\alpha}{2}} u ||_{\H_K} \leq A_1$, so that
\begin{equation}
    \notag
    || \left(\C_{KK}+\lambda I\right)^{-\frac{1}{2}} u || \leq || \left(\C_{KK}+\lambda I\right)^{-\frac{\alpha}{2}}||\cdot || \C_{KK}^{-\frac{1-\alpha}{2}} u || \leq \lambda^{-\frac{\alpha}{2}}\cdot A_1
\end{equation}
as desired.
\end{proof}

\subsection{Concentration inequalities}
\begin{theorem}
\label{bernstein-matrix}
~\cite[Theorem 27]{fischer2020sobolev} 
Let $(\Omega, \mathcal{B}, P)$ be a probability space, $\H$ be a separable Hilbert space and $X: \Omega \rightarrow \mathrm{HS}(H;H)$ be a random variable with self-adjoint values. Furthermore, assume that $||X||_{F} \leq B, P-a.s.$ and $V$ be a positive semi-definite matrix with $\mathbb{E}_{P}\left(X^{2}\right) \preccurlyeq V$, i.e. $V-\mathbb{E}_{P}\left(X^{2}\right)$ is positive semi-definite. Then, for $g(V):=\log \left(2 e \operatorname{tr}(V)||V||^{-1}\right), \tau \geq 1$, and $n \geq 1$, the following concentration inequality is satisfied
$$
P^{n}\left(\left(\omega_{1}, \ldots, \omega_{n}\right) \in \Omega^{n}:||\frac{1}{n} \sum_{i=1}^{n} X\left(\omega_{i}\right)-\mathbb{E}_{P} X(\omega)|| \geq \frac{4 \tau B g(V)}{3 n}+\sqrt{\frac{2 \tau||V|| g(V)}{n}}\right) \leq 2 e^{-\tau} .
$$
\end{theorem}

\begin{theorem}
\label{bernstein}
~\cite[Theorem 26]{fischer2020sobolev} 
Let $(\Omega, \mathcal{B}, P)$ be a probability space, $H$ be a separable Hilbert space, and $\xi: \Omega \rightarrow H$ be a random variable with
$$
\mathbb{E}_{P}||\xi||_{H}^{m} \leq \frac{1}{2} m ! \sigma^{2} L^{m-2}
$$
for all $m \geq 2$. Then, for $\tau \geq 1$ and $n \geq 1$, the following concentration inequality is satisfied
$$
P^{n}\left(\left(\omega_{1}, \ldots, \omega_{n}\right) \in \Omega^{n}:||\frac{1}{n} \sum_{i=1}^{n} \xi\left(\omega_{i}\right)-\mathbb{E}_{P} \xi||_{H}^{2} \geq 32 \frac{\tau^{2}}{n}\left(\sigma^{2}+\frac{L^{2}}{n}\right)\right) \leq 2 e^{-\tau}
$$
\end{theorem}

The following theorem shows that the regularized covariance $\C_{KK}+\lambda I$ can be estimated with small error when $\lambda$ is above a certain threshold. Although it is well-known \citep{fischer2020sobolev,talwai2022sobolev}, we still recall it below for completeness.

\begin{theorem}
    \label{conc-covariance}
    Recall that $\C_{KK}=\E_{P_K}u\otimes u$ and $\hat{\C}_{KK}=\frac{1}{N}\sum_{i=1}^N u_i\otimes u_i$ where $u_i\iidsim P_K$. Suppose that \Cref{asmp-bounded-input-data} holds and $N \gtrsim A_1^2\tau g_{\lambda}\lambda^{-\alpha}$, where $g_{\lambda}=\log\left(2e\mathcal{N}_{P_K}(\lambda)\frac{||\C_{KK}||+\lambda}{||\C_{KK}||}\right)$ and $\mathcal{N}_{P_K}(\lambda)=\mathrm{tr}\left( (\C_{KK}+\lambda I)^{-1}\C_{KK}\right)$ is the \textit{effective dimension}, then with probability at least $1-e^{-\tau}$, we have
    \begin{equation}
        \label{conc-covariance-ineq}
        ||\left(\C_{KK}+\lambda I\right)^{-\half}\left(\C_{KK}-\hat{\C}_{KK}\right)\left(\C_{KK}+\lambda I\right)^{-\half}|| \lesssim \sqrt{\frac{A_1^2\tau g_{\lambda}}{N\lambda^{\alpha}}} \leq 0.1.
    \end{equation}
\end{theorem}

\begin{proof}
    Let $X(u) = \left(\C_{KK}+\lambda I\right)^{-\half} u\otimes u \left(\C_{KK}+\lambda I\right)^{-\half}$ where $u\in\H_K$, then the LHS of \Cref{conc-covariance-ineq} can be expressed as $||\frac{1}{N}\sum_{i=1}^N X(u_i)-\E_{u\sim P_K} X(u)||$. We hope to apply \Cref{bernstein-matrix} and start with verifying the assumptions.

    Since $\E_{P_K}X = \left(\C_{KK}+\lambda I\right)^{-\half} \C_{KK} \left(\C_{KK}+\lambda I\right)^{-\half} $ and $||X(u)|| = ||X(u)||_F = ||\left(\C_{KK}+\lambda I\right)^{-\half} u||^2 \leq A_1^2||\left(\C_{KK}+\lambda I\right)^{-\frac{1-\alpha}{2}}||^2 \lesssim A_1^2\lambda^{-\alpha}$, so that there exists $V = \O\left(\lambda^{-\alpha} \left(\C_{KK}+\lambda I\right)^{-\half} \C_{KK} \left(\C_{KK}+\lambda I\right)^{-\half}\right)$ such that $\E_{P_K}X^2  \preccurlyeq V$. It's easy to see that $||V||\lesssim\lambda^{-\alpha}$ and $\mathrm{tr}(V)\lesssim \mathcal{N}_{P_K}(\lambda)$. The conclusion then follows from \Cref{bernstein-matrix} with $B = \O(\lambda^{-\alpha})$ and $g(V)=g_{\lambda}$.
\end{proof}

\begin{corollary}
    \label{conc-covariance-cor}
    Under the notations and assumptions of \Cref{conc-covariance}, there exists a constant $C_1>0$ with probability $\geq 1-e^{-\tau}$ we have
    \begin{equation}
        \begin{aligned}
            ||\left(\C_{KK}+\lambda I\right)^{\half}\left(\hat{\C}_{KK}+\lambda I\right)^{-1}\left(\C_{KK}+\lambda I\right)^{\half}|| \leq C_1 .
        \end{aligned}
    \end{equation}
\end{corollary}

\begin{proof}
   By \Cref{conc-covariance} we have
    \begin{equation}
        \notag
        \begin{aligned}
            &\quad ||\left(\C_{KK}+\lambda I\right)^{\half}\left(\hat{\C}_{KK}-\C_{KK}\right)^{-1}\left(\C_{KK}+\lambda I\right)^{\half}|| \\
            &= || \left(I - (\C_{KK}+\lambda I)^{-\half} (\C_{KK}-\hat{\C}_{KK}) (\C_{KK}+\lambda I)^{-\half}\right)^{-1}|| \\
            &\leq 2
        \end{aligned}
    \end{equation}
    with probability $\geq 1-e^{-\lambda}$, as desired.
\end{proof}

\end{document}